\documentclass[acmsmall,screen]{acmart}

\setcopyright{rightsretained}
\acmPrice{}
\acmDOI{10.1145/3563329}
\acmYear{2022}
\copyrightyear{2022}
\acmSubmissionID{oopslab22main-p383-p}
\acmJournal{PACMPL}
\acmVolume{6}
\acmNumber{OOPSLA2}
\acmArticle{166}
\acmMonth{10}

\usepackage{placeins}
\usepackage{tabularx,multirow}
\newcolumntype{Y}{>{\centering\arraybackslash}X}
\usepackage{amsmath}
\usepackage{amsthm}

\usepackage{booktabs}   
\usepackage{subcaption} 

\usepackage[capitalize]{cleveref}
\usepackage{algorithm}
\usepackage{algpseudocode}
\usepackage{svg}

\algrenewcommand\algorithmicrequire{\textbf{Input:}}
\algrenewcommand\algorithmicensure{\textbf{Output:}}

\DeclareUnicodeCharacter{3B1}{\ensuremath{\alpha}}
\newcommand{\ourname}[0]{\texorpdfstring{\textsc{αNAS}}}

\newcommand{\lo}[0]{\texttt{o}}
\newcommand{\lx}[0]{\texttt{x}}
\newcommand{\lm}[0]{\texttt{m}}
\newcommand{\la}[0]{\texttt{a}}

\DeclareMathOperator*{\argmin}{\arg\min}

\begin{document}

\title{Neural Architecture Search using Property Guided Synthesis}

\author{Charles Jin}
\authornote{Work performed while at Google Research.}
\affiliation{
  \department{CSAIL}
  \institution{MIT}
  \city{Cambridge}
  \state{MA}
  \country{US}
}
\email{ccj@csail.mit.edu}

\author{Phitchaya Mangpo Phothilimthana}
\affiliation{
  \department{Google Research}
  \institution{Google}
  \city{Mountain View}
  \state{CA}
  \country{US}
}
\email{mangpo@google.com}

\author{Sudip Roy}
\authornotemark[1]
\affiliation{
  \institution{Cohere}
  \city{Palo Alto}
  \state{CA}
  \country{US}
}

\begin{abstract}
Neural architecture search (NAS) has become an increasingly important tool within the deep learning community in recent years, yielding many practical advancements in the design of deep neural network architectures. However, most existing approaches operate within highly structured design spaces, and hence (1) explore only a small fraction of the full search space of neural architectures while also (2) requiring significant manual effort from domain experts. In this work, we develop techniques that enable efficient NAS in a significantly larger design space. In particular, we propose to perform NAS in an abstract search space of program properties. 
Our key insights are as follows: (1) an abstract search space can be significantly smaller than the original search space, and (2) architectures with similar program properties should also have similar performance; thus, we can search more efficiently in the abstract search space.
To enable this approach, we also introduce a novel efficient synthesis procedure, which performs the role of concretizing a set of promising program properties into a satisfying neural architecture. We implement our approach, \ourname{}, within an evolutionary framework, where the mutations are guided by the program properties. Starting with a ResNet-34 model, \ourname{} produces a model with slightly improved accuracy on CIFAR-10 but 96\% fewer parameters. 
On ImageNet, \ourname{} is able to improve over Vision Transformer (30\% fewer FLOPS and parameters), ResNet-50 (23\% fewer FLOPS, 14\% fewer parameters), and EfficientNet (7\% fewer FLOPS and parameters) without any degradation in accuracy.
\end{abstract}

\ccsdesc[500]{Software and its engineering~Search-based software engineering}
\ccsdesc[300]{Computing methodologies~Neural networks}


\keywords{Neural Architecture Search, Program Synthesis, Abstract Interpretation}

\maketitle

\section{Introduction}
\label{sec:intro}

Neural architecture search (NAS) is a class of techniques that aims to automate the process of designing neural network architectures \cite{zoph2017neural}. Specifically, rather than relying on fully hand-designed neural architectures, which is time-consuming and subject to human biases and failures, NAS proposes to automatically search for neural architectures within a search space. These approaches have yielded significant improvements in both raw accuracy and navigating the trade-off with other metrics of practical interest such as parameter count and latency.
 
Despite many recent successes, however, the general NAS problem presents many key technical challenges that have yet to be fully addressed. First, the search space of deep neural networks is prohibitively large. Even if the high-level structure of a network is fixed, one still needs to select between operation variants (e.g., various types of convolutions) and parameter settings (e.g., filter sizes, layer widths, layer depths). Second, good architectures are extremely sparse, i.e., a randomly constructed network is unlikely to perform well. Finally, evaluating candidate architectures is also computationally expensive, as evaluation generally involves training the proposed network to convergence on a target dataset from scratch.

To contend with these challenges, most existing approaches are specialized for highly structured search spaces. For instance, a common strategy is to first manually specify a fixed architecture, then use NAS to tune selected parameters of the operators (e.g., depths and filter sizes) \cite{efficientnet,efficientnet-v2,nasnet,liu2018hierarchical,progressive-nas,he2018amc,li2020gan}. It has been observed that such restricted search spaces also tend to bypass the second challenge, i.e., even randomly selected networks perform quite well \cite{li2020random,yu2019evaluating,bender2020can}. 

However, imposing such structure comes at a direct cost of the expressiveness of the search space, and generalizing prior approaches to unstructured search space is an open problem. Indeed, while existing approaches have delivered impressive results when used to tune existing architectures, to date, major architectural changes (AlexNet \cite{alexnet}, ResNet \cite{he2015deep}, Transformers \cite{transformer}, etc.) have still been driven by the insight of human researchers.

More recently, several new approaches to NAS have been proposed that avoid the need to define a highly structured search space by iteratively mutating a randomly selected subgraph inside a neural network architecture, using an evolutionary search mechanism \cite{real2019regularized,automl-zero,so2021primer}. 
Given the number of possible operations, each with different functional and performance characteristics,
directly generating a random subgraph is unlikely to perform well in place of the original subgraph. To counteract this, prior approaches limit the scope of their mutations to only small changes.
For instance, Primer \cite{so2021primer} uses mutations such as changing a single parameter in a random operation, or swapping two nearby operations. However, restricting the mutations to small changes requires many iterations to discover high-quality, novel architectures, and in many cases, the search may struggle to escape local minima.

\subsection{Our Approach}
\label{sec:overview}

In this work, we present a new evolutionary approach to NAS. Prior methods evolve new architectures by applying small stochastic mutations directly on the computation graphs of deep neural networks. Each such mutation produces a new architecture, which is evaluated for accuracy, then inserted into the population as a potential starting point for future mutations.

In contrast, our approach is able to progress through the search space significantly faster by making high-quality, substantial changes to the architecture in each mutation step. To achieve this, we propose to perform the mutations in an \emph{abstract space of program properties}. We identify the following objectives that guide the design of our program properties:
\begin{enumerate}
    \item By mapping multiple concrete architectures to the same set of program properties, the abstract search space can be significantly smaller in size than the original search space.
    \item Architectures that have similar program properties should also have similar performance. This enables search algorithms which rely on an assumption of locality in the search space (e.g, hill-climbing or evolutionary algorithms) to be applied to the unstructured search space.
\end{enumerate}

Our approach, named \ourname{}, begins by selecting a random subgraph of the computation graph for mutation. Then, we infer a set of program properties for the selected subgraph, and apply stochastic mutations on the properties of the subgraph. Finally, we synthesize a new subgraph that satisfies the mutated properties, and insert the new subgraph in place of the old subgraph to produce a new architecture. 

\begin{figure}[t]
    \centering
    \includegraphics[width=.8\linewidth,trim={2.5cm 3cm 2cm 0cm}]{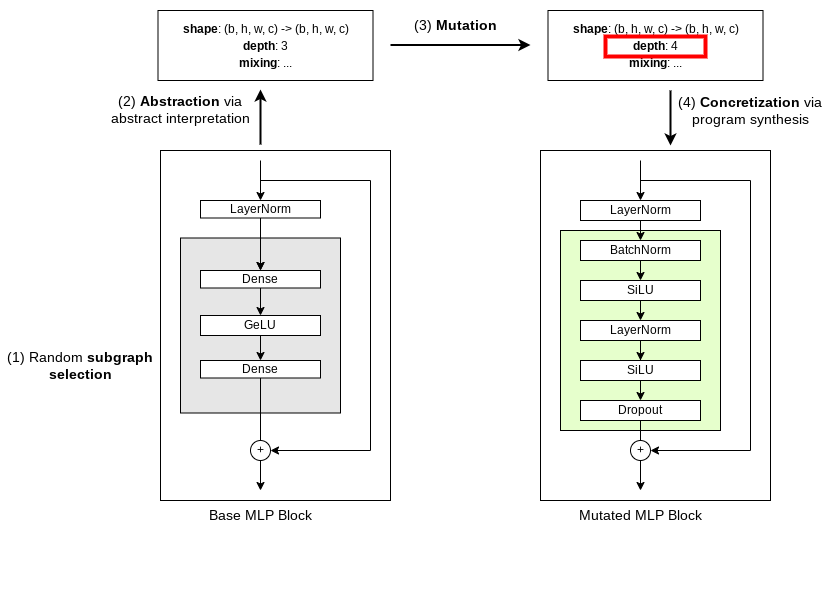}
    \caption{Our approach proposes a new neural network architecture by (1) randomly selecting a subgraph of an existing architecture, (2) inferring program properties of the subgraph, (3) mutating the program properties, and (4) synthesizing a new subgraph that satisfies the mutated program properties. We apply these four steps iteratively. This particular mutation is found during our ImageNet experiment when starting from a Vision Transformer architecture. A descendant of this model (consisting of 4 total mutations) ultimately decreases FLOPS and parameter count by 28\% and 30\% respectively, while slightly increasing accuracy.
    }
    \label{fig:overview}
\end{figure}

\cref{fig:overview} illustrates an example of a mutation synthesized by \ourname{} during the evolution of the Vision Transformer architecture \cite{dosovitskiy2021image} for image classification. To perform this mutation, we first select a random subgraph from the computation graph of the \emph{parent} architecture (step 1, \cref{fig:overview}; the selected subgraph is enclosed in a gray box). Next, we abstract the subgraph into a set of \emph{program properties} (step 2, \cref{fig:overview}). Since deep neural network graphs can be computationally expensive to execute, one crucial requirement is that our properties can be inferred \emph{statically}, that is, without executing the computations on real data. Of course, it is also important that the abstract properties capture semantically relevant information about the underlying computation. For instance, our most basic property is the \emph{depth} of the subgraph in terms of the number of alternating linear and nonlinear layers; here, the selected subgraph has a depth of 3.

We then apply stochastic mutations at the level of abstract program properties to obtain new program properties (step 3, \cref{fig:overview}; the mutated depth property is marked by a red box). The mutated properties are concretized back into a subgraph (step 4, \cref{fig:overview}). Finally, the new subgraph is inserted in place of the original subgraph (marked by the green box), yielding a new, mutated architecture (called a \emph{child}), which is ready to be evaluated and added to the population for the next round of evolution.
Note that our search space contains over 1402 possible operations for each node. To satisfy the new depth property of 4, the synthesized subgraph must contain at least 4 operations, yielding \emph{over 3.8 billion legal subgraphs}. In order to perform the concretization efficiently, in \cref{sec:synthesis} we propose a novel program synthesis algorithm. Using our techniques, synthesizing the subgraph in \cref{fig:overview} (with 5 nodes) took only 88 seconds---a fraction of the time needed to evaluate the resulting architecture (which involves training the model on a dataset). 

The mutation displayed in \cref{fig:overview} is ultimately a part of the 278th individual (evolved using 4 total mutations) that improves the baseline architecture by decreasing FLOPS by 28\% and parameter count by 30\%, while slightly increasing accuracy. Note that this single mutation is equivalent to 8 steps in the concrete space using previous techniques like Primer.
Hence, the primary benefit of our approach is that a single mutation in the abstract space can result in a new architecture that would otherwise require many individual concrete mutations to achieve. Intuitively, by requiring the synthesized subgraph to satisfy similar properties to the original subgraph, our search is biased toward high quality replacements. 
By enabling larger mutations, our method is able to explore new architectures far more quickly, and requires substantially fewer iterations of evolutionary search compare to prior approaches; in our experiments, we only require evolving 800 individuals to produce a population of novel, high quality architectures for image classification on the ImageNet dataset, whereas Primer evolves nearly 25,000 individuals on the task of language modeling.

\subsection{Our Contributions}

Our first contribution is a set of program properties defined over arbitrary neural architectures, which constitute the abstract search space in which the search is performed. In particular, we introduce the \emph{shape}, \emph{depth}, and \emph{mixing} properties for neural networks in \cref{sec:properties}. We also describe efficient algorithms for inferring these program properties; crucially, we infer these properties statically, which enables the inference to occur exclusively on a CPU host, decoupling the costly evaluation process from the rest of the search procedure.

Our second contribution is an efficient synthesis procedure, which accepts a set of program properties and returns a satisfying neural architecture. 
In general, inverting an abstraction function is nontrivial; in the worst case, the only technique which is guaranteed to succeed is a brute-force search that enumerates all possible architectures and returns the first satisfying architecture. However, by carefully crafting our synthesis procedure to leverage a notion of distance to the program properties, our procedure is able to synthesize a satisfying architecture exponentially faster than the na\"ive enumerative strategy.

To empirically validate our approach, we implement our techniques within a basic evolutionary search and evaluate performance on image classification tasks.\footnote{Code is available at \url{https://github.com/google-research/google-research/tree/master/abstract_nas}.} Starting with a ResNet-34 model on the CIFAR-10 dataset, \ourname{} produces a model with slightly improved accuracy but 96\% fewer parameters. 
On the ImageNet dataset, \ourname{} is able to improve over Vision Transformer (30\% fewer FLOPS and parameters), ResNet-50 (23\% fewer FLOPS, 14\% fewer parameters), and EfficientNet (7\% fewer FLOPS and parameters) without any degradation in accuracy. When compared against the current state-of-the-art search mechanisms from Primer \cite{so2021primer} and AutoML-Zero \cite{automl-zero}, \ourname{} discovers significantly more Pareto optimal models.

\section{Problem Formulation}

This section begins with a brief introduction of deep neural network architectures, then describes how we formulate the problem of searching for new neural architectures as a program synthesis task by treating deep neural networks as programs. The remaining subsections formalize the setting for our synthesis algorithm.

\subsection{Deep Neural Network Architectures}

\begin{figure}[t]
    \centering
    \includegraphics[width=.8\linewidth]{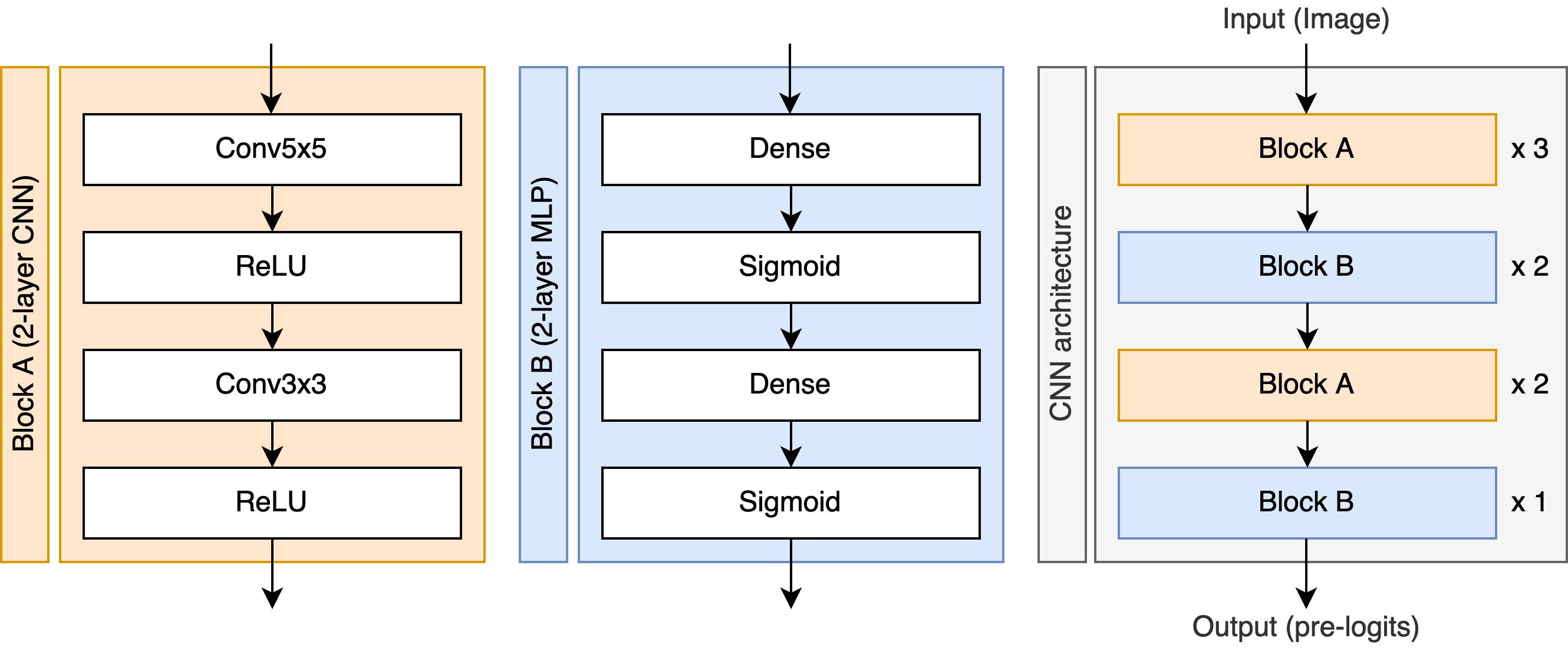}
    \caption{A toy example of a convolutional neural network architecture for image classification. The final architecture (right) is composed of two different types of blocks (left, middle), with repetitions annotated on the far right. Not shown is the classification head, which projects the ``Output (pre-logits)'' through a fully-connected layer to the number of classes followed by a softmax layer, yielding the final class probabilities.}
    \label{fig:cnn_arch}
\end{figure}

Deep neural networks (DNNs) are a family of models composed of layers of artificial neurons. Here we describe the most basic architecture, a simple feedforward multilayer perceptron (MLP). The input to the first layer is often some sensory data, such as an image; for subsequent layers, the inputs are the outputs of the neurons in the previous layer. Each neuron first computes a weighted combination of its inputs, where the weights are learnable parameters. The result is then passed through an element-wise nonlinearity, called an activation function; this becomes the output of the neuron, to be consumed by the next layer. The output of the final layer of neurons is the output of the neural network. Such an architecture is ``deep'' when it contains many layers alternating linear functions with nonlinear activations. Each layer is differentiable, so the entire network is trainable using gradient descent via the back-propagation algorithm. Depending on the problem domain, some common variations include selecting different nonlinearities, constraining the inputs of a neuron to a subset of the previous layer (e.g., convolutional neural networks, which use convolutions), or supporting additional types of operations (e.g.,  normalization layers).

\paragraph{Representing DNNs}
Since each neuron in a layer has the same functional form, we can concisely represent each layer of an MLP as a sequence of operations (e.g., a dense linear layer followed by a sigmoid activation); the complete architecture is represented by the full sequence of operations over every layer. Additionally, modern deep neural network architectures are often composed of repeated blocks, where each block consists of a fixed sequence of operations. These blocks are then stacked in sequence to produce the final architecture. \cref{fig:cnn_arch} displays a toy example of a convolutional neural network architecture, composed of several repeated blocks. More generally, we can represent more complex (i.e., nonfeedforward) architectures as a computation graph, where nodes are operations and the directed edges represent dataflow relationships (cf. \cref{fig:vit_mutation}).

\paragraph{From Neural Architecture Search to Program Synthesis}
NAS is the problem of automatically searching for novel neural architectures, given some performance objectives. For instance, we may want to search for CNN architectures that maximize accuracy on the CIFAR-10 dataset. By representing neural architectures as computation graphs, this paper casts the problem of neural architecture search as a program synthesis task. More specifically, we treat computation graphs as programs, and perform the search by synthesizing new programs. The remainder of this section presents a formal setting for synthesis based on \emph{program properties}. In \cref{sec:properties}, we propose program properties that can be leveraged to guide the synthesis procedure toward architectures with better performance, thereby enabling us to apply program synthesis techniques to NAS.

\subsection{Program Properties} Program properties form the basis of our synthesis problem. We denote the arbitrary, fixed space of programs as $\mathcal{P}$.

\begin{definition}
\label{def:property}
A \textbf{program property} $\Pi = (\mathcal{V}, \le, \alpha)$ over the space of programs $\mathcal{P}$ consists of a set $\mathcal{V}$ of \emph{program property values}, with a partial order relation $\le$ over $\mathcal{V}$, and an \emph{abstraction function} $\alpha : \mathcal{P} \mapsto \mathcal{V}$ that maps programs to property values.
\end{definition}

\begin{definition}
A program $p \in \mathcal{P}$ \textbf{satisfies} a program property value $v \in \mathcal{V}$, denoted $p \models v$, if $\alpha(p) \ge v$.
\end{definition}

For instance, the previous section introduces a depth property for deep neural networks. This property takes values $\mathcal{V}$ in the nonnegative integers, where $\le$ is the usual relation on integers, and $\alpha$ maps a subgraph to its depth. For a feedforward neural network, the depth is simply the number of alternating linear and nonlinear operators in the network, and a subgraph satisfies the depth program property value $d$ if its depth is at least $d$.

We can immediately generalize this definition to a setting where we are given a set of program properties $\mathbb{P} = \{\Pi_i = (\mathcal{V}_i, \le_i, \alpha_i)\}_{i = 1}^N$:

\begin{definition}
A program $p \in \mathcal{P}$ \textbf{satisfies} a set of program property values $S = \{v_i \, | \, v_i \in \mathcal{V}_i\}_{i=1}^N$, denoted $p \models S$, if for every $v_i \in S$, we have that $p \models v_i$.
\end{definition}

We usually drop the subscript $i$ of the partial ordering $\le_i$ when it is clear from context.

\subsection{Program Synthesis}

Program transformations are the basic unit of our synthesis algorithm. We formalize the task of program synthesis as the problem of producing a sequence of transformations that transforms some initial program $p_0$ into a program $p^*$ satisfying the desired program properties $S$.

\begin{definition}
A \textbf{program transformation} $t \in \mathcal{T} : \mathcal{P} \mapsto \mathcal{P}$ is a map from programs to programs.
\end{definition}

\begin{definition}
The synthesis task $(p_0, v)$ is \textbf{feasible} if there exists some sequence $t_1, \ldots, t_n \in \mathcal{T}$ such that $(t_n \circ \cdots \circ t_1) (p_0) \models S$.
\end{definition}

In this work, we will always take the initial program $p_0$ to be the identity program: $\text{id}(x) = x$ for all inputs $x$. We will also assume that $\mathcal{T}$ is freely generated (via composition) by a known, finite set $E$ of \textbf{primitive transformations}, i.e., $\mathcal{T} = E^*$. For synthesizing subgraphs of deep neural networks, the primitive transformations $E$ consist of appending one of the basic operations to the partial subgraph. Given any set of program properties values $S$, our synthesizer incrementally appends operations \emph{without backtracking}, until the resulting program satisfies the given properties. The key to ensuring progress is defining an appropriate notion of \emph{distance} from programs to the desired property, such that we can always decrease the distance by selecting a transformation from a small fixed subset $T \subset \mathcal{T}$. Hence, each step of the synthesis algorithm simply enumerates over the transformations in $T$, and selects a transformation that is guaranteed to decrease the distance. This process is repeated until the subgraph achieves a distance of 0, satisfying the target property.

For instance, to satisfy the depth property in \cref{fig:overview}, we can simply alternate between appending linear and nonlinear operations. Clearly, this strategy can always increase the depth of the subgraph until the property is satisfied. Our main insight is that this seemingly simple procedure also applies to program properties with more complex semantics. As each step of synthesis progressively brings the program closer to satisfying the property, we call our technique \textbf{progressive program synthesis}. Crucially, the runtime of this algorithm is \emph{linear} in the length of the synthesized subgraph, whereas the number of possible subgraphs is \emph{exponential} in the length.

\subsection{Property Inference}

$\alpha$NAS performs property inference ($\alpha$) during the abstraction and synthesis phases (steps 2 and 4 in \cref{fig:overview}, respectively). During abstraction, we infer the properties of the selected subgraph only once. During synthesis, however, our algorithm relies on inferring the properties of each partially synthesized program after applying some transformation $t$.

To infer program properties efficiently, we extend the semantics of $\alpha$ from programs $p \in \mathcal{P}$ to transformations $t \in \mathcal{T}$. The idea is that if we already know $\alpha(p) = v$, instead of computing $\alpha(t(p))$ from scratch, we can leverage the partial result $v$ and compute the effects of $t$ on $v$ abstractly.
We formalize this intuition using the theory of \emph{abstract interpretation} for transformations. In particular, when computing the effects of a program transformation $t \in \mathcal{T}$ on a program property value $v$, we need to preserve the soundness of the $\models$ relationship in the following sense:

\begin{definition}
\label{def:abstract_interpretation}
Given a program property $\Pi = (\mathcal{V}, \le, \alpha)$, an \textbf{abstract interpretation} of a program transformation $t \in \mathcal{T}$ is a map $t^\alpha: \mathcal{V} \mapsto \mathcal{V}$ such that $t^\alpha(u) = v$ implies $\forall p \models u$, $t(p) \models v$.
\end{definition}

Henceforth, we use the superscript $\alpha$ to identify transformations $t$ over concrete subgraphs with an abstract interpretation $t^\alpha$ over program properties. \cref{fig:abstract_interpretation} illustrates the key soundness property of abstract interpretation as a commutative diagram.

\begin{figure}[t]
    \centering
    \includegraphics[width=.45\linewidth,trim={0cm 6cm 13.5cm 0cm},clip]{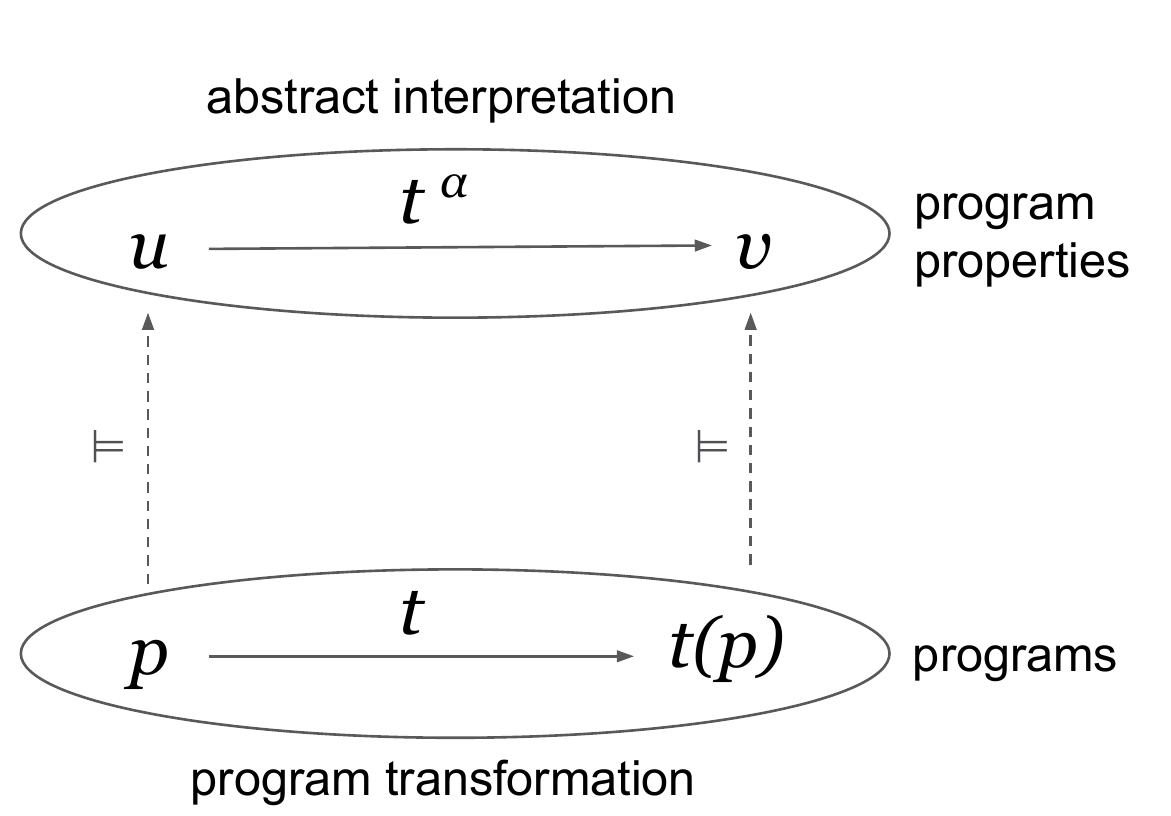}
    \caption{The key soundness property of abstract interpretation: $t^\alpha(u) = v$ implies that $\forall p \models u$, $t(p) \models v$. Superscript $\alpha$ denotes the abstract interpretation of transformation $t$ with respect to a program property $\Pi$.} 
    \label{fig:abstract_interpretation}
\end{figure}

\section{Program Properties for Neural Architecture Search}
\label{sec:properties}

In this section, we present the set of program properties used by \ourname{} to guide the neural architecture search. Before identifying the specific properties, we first elaborate on our main considerations in defining such properties for NAS. 

As outlined in \cref{sec:overview}, a key step in evolutionary NAS is how architectures are mutated over the course of the search. 
Thus, it is essential to generate mutations that, on expectation, yield a better child architectures. 
Since our approach involves inferring the properties of the original subgraph, mutating these properties, then synthesizing a new subgraph from the mutated properties, choosing the right set of properties is critical. 

Specifically, our objectives in designing the program properties are threefold:
\begin{enumerate}
    \item The abstract space is smaller because multiple concrete architectures are mapped to a single set of properties. However, the abstraction function should also be well-aligned with our search objective, i.e., architectures that are mapped to the same point in the abstract space should have similar performance characteristics.
    \item Since each point in abstract space represents multiple concrete architectures, small steps in the abstract space represent large steps in the original concrete space; but for the abstraction to be useful, we want small changes in the abstract space to preserve at least some of the performance characteristics in the concrete space---otherwise, there is no structure for the search mechanism to leverage.
    \item Since the substitution step involves both inferring program properties for a candidate subgraph, and synthesizing an alternate subgraph based on mutated properties, both of these steps must be computationally efficient.
\end{enumerate}

In particular, the first two points are desiderata that relate the program properties to the semantics of deep neural network architectures. To this end, we propose the following characterization of what constitutes a deep neural network:

\begin{definition}
\label{def:mlp}
A \textbf{deep neural network} is a function containing alternating layers of linear and nonlinear functions.
\end{definition}

The linear functions include layers such as convolutional, dense, and fully-connected layers, while the nonlinear functions consist of activation functions such as ReLU, sigmoid or tanh. This description sufficiently captures the simplest type of deep neural network: the feedforward MLP. It follows that any abstract search space for DNNs should be expressive enough to capture this key characteristic. For instance, our properties should be able rule out the following degenerate architectures, which are valid programs in our search space, but do not satisfy the basic definition above:
\begin{enumerate}
    \item a sequence of activation functions, without any linear layers in between;
    \item a sequence of linear layers, without any activations in between;
    \item an alternating sequence of pooling layers and activation functions.
\end{enumerate}

Note that the last example technically consists of alternating layers of linear and nonlinear functions, but lacks expressivity from a functional perspective, as it has no trainable parameters and therefore can be used to represent only a very restricted class of functions.
This motivates our approach, which is to capture this characterization from the perspective of mathematical functions. In the remainder of this section, we describe formally the subgraph properties we propose in this work, as well as the associated efficient inference algorithms based on abstract interpretation. For clarity of presentation, we assume that the subgraphs have only a single input and output; to generalize our definitions to subgraphs with multiple inputs or multiple outputs, we simply compute the property for every input-output pair.

\subsection{Mixing Property}

The mixing property captures key information about the expressivity of a subgraph \emph{as a linear operator}. Intuitively, if a subgraph consists purely of linear operators (and hence is linear as a whole), then each element of the output is a linear combination of elements in the input; the mixing property summarizes information about which elements of the input have nonzero coefficients in the linear combination.

\paragraph{An example.} Consider a \mbox{3~$\times$~3} convolutional layer:
\begin{align}
    O[ b, c_o, w_o, h_o] = \sum_{c_i=0}^{C_i}\sum_{k_h=-1}^{1}\sum_{k_w=-1}^{1} I[b, c_i, w_o + k_w, h_o + k_h] * K[c_o, c_i, k_w+1, k_h+1]
\end{align}
and a dense layer:
\begin{align}
    O[b, c_o, w_o, h_o] = \sum_{c_i=0}^{C_i} I[b, c_i, w_o, h_o] * W[c_o, c_i]
\end{align}
where $O$ is the output tensor, $I$ is the input tensor, $K$ is the kernel of the convolution, and $W$ is the weight matrix of the dense layer.

Intuitively, we want to capture the fact that the \mbox{3~$\times$~3} convolutional layer is strictly more expressive than a dense layer, since the dense layer is just a special case of the \mbox{3~$\times$~3} convolution where only the middle element of the convolutional filter is nonzero, i.e., any dense layer can be written as a convolutional layer (but not vice-versa). Another way to express this observation is that every output element of a convolutional layer depends on a full slice of the input channel dimension, as well as \emph{many} elements of the input spatial dimensions; however every output element of a dense layer depends on a full slice of the input channel dimension, but only a \emph{single} element of the input spatial dimension. Similarly, fully connected layers are strictly more expressive than both convolutional and dense layers.
We formalize these intuitions in the mixing property, which is composed of two subproperties.

\subsubsection{Pairing subproperty}

The pairing subproperty captures coarse-grained information about which \emph{dimensions} of the input tensor contribute to which \emph{dimensions} of the output tensor. For instance, consider a convolutional layer applied to an input with two spatial dimensions and one channel dimension (e.g., a standard 2D image with three color channels). Then the channel dimension of the output depends on the channel dimension of the input; however, the channel dimension of the output does not depend on the spatial dimensions of the input, because the convolutional kernel has a limited receptive field in the spatial dimensions.

\begin{figure}[t]
     \centering
     \begin{subfigure}[b]{0.5\textwidth}
         \centering
         \includegraphics[width=\textwidth]{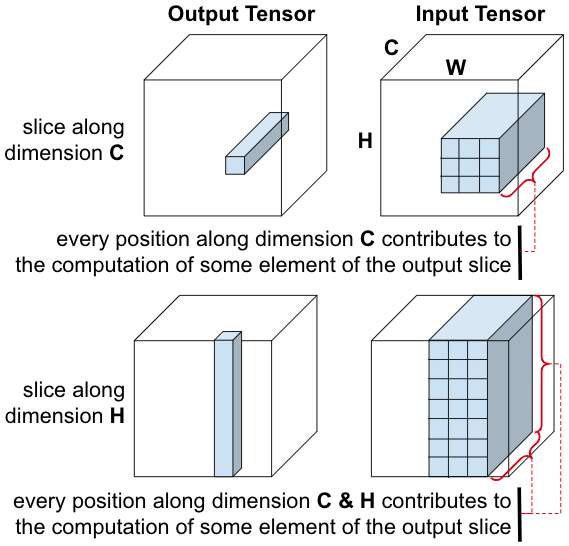}
         \caption{Pairing dimensions. Top: the input dimension C is paired with the output dimension C. Bottom: the input dimensions C and H dimensions get paired with the output dimension H.}
         \label{fig:mixing}
     \end{subfigure}
     \hfill
     \begin{subfigure}[b]{0.45\textwidth}
         \centering
         \includegraphics[width=\textwidth]{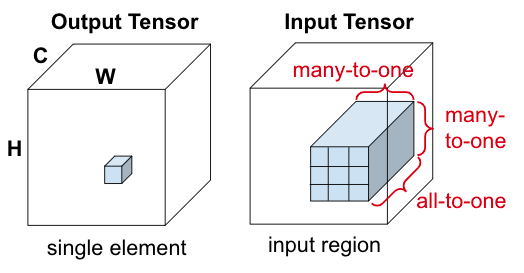}
         \caption{Locality. The input dimension C has all-to-one locality with the paired output dimensions. The input dimensions H and W have many-to-one localities with the paired output dimensions.}
         \vspace{4em}
         \label{fig:locality}
     \end{subfigure}
     \caption{Determining the mixing property of a \mbox{3~$\times$~3} convolutional layer when applied to a 2D image (H x W) with a channel (C) dimension.}
\end{figure}

More explicitly, given an input tensor and an output tensor, we say the operation pairs a dimension of the input and a dimension of the output if, given a fixed slice of the output along the output dimension, there is at least one element for each position in the input dimension which contributes to the computation of the output slice.

The formal definition is as follows. Let $f$ be a tensor-valued function which maps an input tensor $A$ of size $(a_1, \ldots, a_n)$ to an output tensor $B$ of size $(b_1, \ldots, b_m)$. In other words, for any element $B[o_1, \ldots, o_m]$, where $1 \le o_k \le b_k$ for all $1 \le k \le m$, we have a real-valued function $f_{o_1, \ldots, o_m}$ over $n$-dimensional tensor inputs $A$ such that
\begin{align}
    B[o_1, \ldots, o_m] = f_{o_1, \ldots, o_m}(A)
\end{align}

\begin{definition}
An element $A[i_1, \ldots, i_n]$ (where $1 \le i_k \le a_k$ for all $1 \le k \le n$) of the input tensor \textbf{contributes to} the computation of the element $B[o_1, \ldots, o_m]$ of the output tensor if
\begin{align}
    \dfrac{\partial f_{o_1, \ldots, o_m}}{\partial A[i_1, \ldots, i_n]}(A) \ne 0
\end{align}
i.e., the partial derivative \emph{as a function} of $A$ is not identically zero.
\end{definition}

Note that if $f$ is a linear function, then the formula for $B[o_1, \ldots, o_m]$ can be expressed as a linear combination over the elements of $A$, and the partial derivative of $f_{o_1, \ldots, o_m}$ with respect to $A[i_1, \ldots, i_n]$ is just the coefficient of $A[i_1, \ldots, i_n]$ in the linear combination.

\begin{definition}
The \textbf{preimage} of the function $f$, denoted as $f^{-1}$, is a map from sets of output elements to sets of input elements, and is defined as follows. Let $S = \{B[o_{1,j}, \ldots, o_{m,j}]\}_{j \in J}$ be a set of output elements, indexed by the set $J$. If $S$ consists of a single element $B[o_1, \ldots, o_m]$, then $f^{-1}(S)$ is the set of all input elements that contribute to the computation of $B[o_1, \ldots, o_m]$. Otherwise,
\begin{align}
    f^{-1}(S) = \bigcup_{j \in J} f^{-1}(\{B[o_{1,j}, \ldots, o_{m,j}]\})
\end{align}
\end{definition}

\begin{definition}
Given a tensor $B$ of shape $(b_1, \ldots, b_m)$, a \textbf{slice} along the $k^{th}$ dimension, specified by $m-1$ indices $\{o_1, \ldots, o_{k-1}, o_{k+1}, \ldots, o_m\}$, is the set of elements
\begin{align}
    B[o_1, \ldots, o_{k-1}, \bullet, o_{k+1}, \ldots, o_m] = \bigcup_{o=1}^{b_k} B[o_1, \ldots, o_{k-1}, o, o_{k+1}, \ldots, o_m]
\end{align}
\end{definition}

\begin{definition}
The function $f$ \textbf{pairs} the $l^{th}$ dimension of the input tensor $A$ with the $k^{th}$ dimension of the output tensor $B$ if there exists a slice $S_B$ of the output tensor along the $k^{th}$ dimension such that the preimage $f^{-1}$ of $S_B$ contains one element in every position along the $l^{th}$ dimension of the input, i.e., for every $1 \le i \le a_l$, there exists $i_1, \ldots, i_{l-1}, i_{l+1}, \ldots, i_n$ such that $A[i_1, \ldots, i_{l-1}, i, i_{l+1}, \ldots, i_n] \in f^{-1}(S_B)$.
\end{definition}

\cref{fig:mixing} illustrates how the pairing subproperty of a convolutional layer is determined for a standard 2D image with an additional channel dimension. We can collect the pairing property over all pairs of input and output dimensions into a matrix in the manner of \cref{table:pairing_dim}, which summarizes the complete pairing subproperty of a convolutional layer for a batch of 2D images.

\begin{figure}[t]
     \centering
\begin{subfigure}[b]{0.48\textwidth}
         \centering
    \footnotesize
    \begin{tabularx}{.6\columnwidth}{c *{6}{Y}}
        \toprule
         & & \multicolumn{4}{c}{In}\\
        \cmidrule{3-6} \
        & & B & H & W  & C \\
        \midrule
        \multirow{4}{*}{Out} 
        & B & 1 & 0 & 0 & 1 \\
        & H & 0 & 1 & 0 & 1 \\
        & W & 0 & 0 & 1 & 1 \\
        & C & 0 & 0 & 0 & 1 \\
        \bottomrule
    \end{tabularx}
    \caption{Pairing. ``1'' indicates a pairing and ``0'' indicates no pairing. For instance, the output batch dimension is paired with the input channel dimension.}
    \label{table:pairing_dim}
\end{subfigure}
\hfill
\begin{subfigure}[b]{0.48\textwidth}
    \centering
    \footnotesize
    \begin{tabularx}{.6\columnwidth}{c *{6}{Y}}
        \toprule
         & & \multicolumn{4}{c}{In}\\
        \cmidrule{3-6} \
        & & B & H & W  & C \\
        \midrule
        \multirow{4}{*}{Out} 
        & B & \lo & \lx & \lx & \la \\
        & H & \lx & \lm & \lx & \la \\
        & W & \lx & \lx & \lm & \la \\
        & C & \lx & \lx & \lx & \la \\
        \bottomrule
    \end{tabularx}
    \caption{Locality + pairing. The entries are all-to-one (\la{}), many-to-one (\lm{}), and one-to-one (\lo{}) locality, or no pairing (\lx{}). For instance, the output H dimension depends on the input H dimension with \lm{} locality.}
    \label{table:pairing_locality}
\end{subfigure}
\caption{Matrices representing the mixing property of a \mbox{3~$\times$~3} convolutional layer with a batch dimension (B).}
\end{figure}

\subsubsection{Locality subproperty}

Consider again a dense layer versus a convolutional layer. The convolutional layer has the same pairing subproperty as the dense layer. The locality subproperty captures a more fine-grained notion of the contributing elements which distinguishes the two operation types. In particular, the pairing subproperty considers the preimage of \emph{slices} to capture the relationships \emph{between} input and output dimensions, whereas the locality subproperty uses the preimage of a single \emph{element} of the output to determine relationships \emph{within} input dimensions.

We distinguish between all-to-one, many-to-one, and one-to-one localities. For instance, a dense layer has a one-to-one locality within the spatial dimensions and an all-to-one locality within the channel dimensions, while a convolution has a many-to-one locality within the spatial dimensions and an all-to-one locality within the channel dimensions. We have the following formal definition:

\begin{definition}
The \textbf{locality} of a function $f$ with respect to the $k^{th}$ dimension of the input is:
\begin{enumerate}
    \item \textbf{all-to-one}, if there exists an element indexed by $\{o_1, \ldots, o_m\}$ of the output tensor $B$ whose preimage $S_A = f^{-1}(\{B[o_1, \ldots, o_m]\})$ contains one element in \textbf{every} position along the $k^{th}$ dimension of the input tensor, i.e., for every $1 \le i \le a_k$, there exists indices $i_1, \ldots, i_{k-1}, i_{k+1}, \ldots, i_n$ such that $A[i_1, \ldots, i_{k-1}, i, i_{k+1}, \ldots, i_n] \in S_A$;
    \item \textbf{many-to-one}, if $f$ is not all-to-one and there exists an output element whose preimage contains elements in \emph{more than one} positions along the $k^{th}$ dimension of the input tensor;
    \item \textbf{one-to-one}, if every element of the output tensor has a preimage that contains elements in \emph{exactly one} position along the $k^{th}$ dimension of the input tensor.
\end{enumerate}
\end{definition}

Clearly, any (non-constant) input dimension is exactly one of all-to-one (\la{}), many-to-one (\lm{}), or one-to-one (\lo{}). \cref{fig:locality} illustrates how the locality subproperty of a convolutional layer is determined. For a given operation, the locality subproperty is computed for each dimension of the input.

\subsubsection{Mixing property}

Similar to the pairing subproperty, we represent the full mixing property as a matrix. However, we replace the paired entries (``1'') with the locality (\lo{}, \lm{}, \la{}) of the corresponding input dimension; entries which are not paired (``0'') are replaced by the additional symbol \lx{}. \cref{table:pairing_locality} displays the full mixing property of a convolutional layer (with an additional batch dimension); a dense layer would have the same matrix, except the many-to-one entries in the spatial dimensions would be replaced by one-to-one entries. Notice that this representation does indeed capture both the pairing and locality subproperties; we chose this representation to simplify the abstract interpretation of the mixing property given by \cref{lemma:linear_abstract_semantics}.

We define the following partial order $\le$ on the mixing property: $P \le Q$ if and only if (1) $P$ and $Q$ have the same number of input and output dimensions and (2) $P \le Q$ element-wise on all input-output pairs, where \lx{} < \lo{} < \lm{} < \la{}. The partial order succinctly captures the desired relationship between dense, convolutional, and fully connected layers, where $<$ corresponds to ``less expressive''.

Note that the locality subproperty cannot differentiate between an identity operation and a transpose operation, and more generally, any operation that simply reshapes its input; for that we must use the pairing subproperty. Conversely, the pairing property does not differentiate between a dense layer and a convolutional layer, while the locality subproperty does. Hence, the two subproperties are complementary, and their combination in the mixing property gives a more complete view of the subgraph as a linear operator.

\paragraph{Concrete Inference}

We first describe how to infer the mixing property of an operation using concrete values. The key is that the ``contributes to'' relationship between input and output elements is defined using gradients; since our domain comes with built-in auto-differentiation, we can leverage the existing functionality to infer the mixing property. More explicitly, for the pairing subproperty, we compute the gradient of a slice of the output tensor with respect to the input tensor by summing along the slice of the output tensor. Any entries in the input tensor with nonzero gradients are said to contribute to the slice. Evaluating the locality subproperty is similar, except that we compute the gradient of the input with respect to a single element of the output.

We additionally exploit the regularity of the individual operations used to construct deep neural network graphs to simplify the inference procedure. Namely, the ``contributes to'' relationship for the primitive operations is equivalent for all slices and elements away from the boundaries; hence, we can efficiently compute the pairing and locality subproperties (or more precisely, a fairly tight lower bound for the subproperties) by evaluating them at only the \emph{center} slice and element, rather than taking the maximum over \emph{all} slices and elements, respectively.

\paragraph{Abstract Interpretation}
\label{para:mixing_abstract_interp}

We next describe an abstract interpretation of the mixing property. Let $\alpha_M$ denote the property inference function which maps concrete subgraphs to their mixing properties. Consider two programs $p$ and $q$. We will treat $q$ as the program transformation (appending $q$ to $p$) for which we want to derive an abstract interpretation, i.e., we would like to infer a lower bound for $\alpha_M(q \circ p)$ given $\alpha_M(p) = u$ without direct access to $p$.

Recall that a mixing property is represented by a matrix with elements \lx{}, \lo{}, \lm{}, \la{}, (cf. \cref{table:pairing_locality}). The abstract interpretation $q^\alpha_M$ can be defined as a matrix multiplication on $u$ and $\alpha_M(q)$.

\begin{lemma}
\label{lemma:linear_abstract_semantics}

$q^\alpha_M(u) := \alpha_M(q) \ \times \  u$ is an abstract interpretation of $q$, where $\times$ is a matrix multiplication with $+$ and $*$ defined over the elements $\{\lx{}, \lo{}, \lm{}, \la{}\}$ as: $y + z = \max(y, z)$;
and $y * z$ is given by the look-up table in \cref{fig:locality_multiply}.

\end{lemma}

We defer the proof to the \cref{appendix:proofs:properties}. Finally, to actually infer the mixing property of a subgraph, we first infer the properties of each concrete operation within the subgraph, and then compose them using the abstract interpretation. To reduce overhead, we also cache the properties for each unique operation by hashing its specification (e.g., operation type and parameter values).

\begin{figure}[t]
\begin{minipage}[b]{0.4\linewidth}
\centering
\footnotesize
        \begin{tabularx}{.5\columnwidth}{c *{6}{Y}}
        \toprule
         & & \multicolumn{4}{c}{z}\\
        \cmidrule{3-6} \
        & & \lx{} & \lo{} & \lm{} & \la{} \\
        \midrule
        \multirow{4}{*}{y} 
        & \lx{} & \lx{} & \lx{} & \lx{} & \lx{} \\
        & \lo{} & \lx{} & \lo{} & \lm{} & \la{} \\
        & \lm{} & \lx{} & \lm{} & \lm{} & \la{} \\
        & \la{} & \lx{} & \la{} & \la{} & \la{} \\
        \bottomrule
    \end{tabularx}
\caption{Definition of $y * z$ over the elements $\{\lx{},\lo{},\lm{},\la{}\}$, representing the locality property, used in \cref{lemma:linear_abstract_semantics}}
\label{fig:locality_multiply}
\vspace{1.5em}
\end{minipage}
\hfill
\begin{minipage}[b]{0.55\linewidth}
\centering
\includegraphics[width=.9\textwidth]{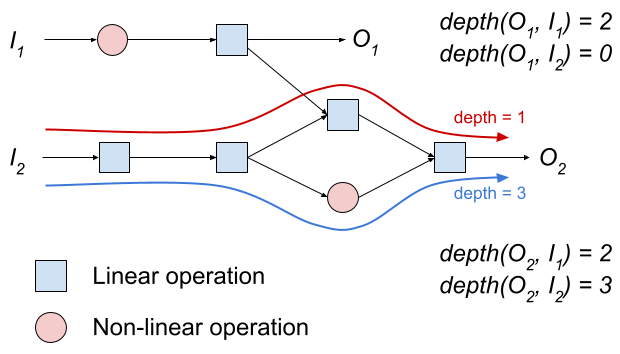}
\caption{Depth property of a subgraph with inputs $I_1, I_2$ and outputs $O_1, O_2$. The depth to $O_1$ from $I_2$ is defined to be 0 since there is no path from $I_2$ to $O_1$. Since there are two paths from $I_2$ to $O_2$, depth($O_2$, $I_2$) is the max of the two paths.} 
\label{fig:depth}
\end{minipage}
\end{figure}

\subsection{Depth Property}

To complement the mixing property, which captures the expressivity of a subgraph as a linear operator, the depth property captures the properties of the subgraph as a \emph{nonlinear} function.
More specifically, given an input and an output of a subgraph, the depth property specifies the maximum depth over all paths from the input to the output, where the depth of a path is defined as the number of alternating linear and nonlinear operations along that path. \cref{fig:depth} gives an example of computing the depth property for a subgraph.

The depth property takes property values in the nonnegative integers. The partial order $\le$ is just the usual ordering on integers. In other words, given a depth property value $d$, any subgraph whose depth property value matches or exceeds $d$ is said to satisfy the depth property.

\paragraph{Concrete Inference}

An operation $f$ is linear if $af(x) + bf(y) = f(ax + by)$
for all scalars $a, b$ and inputs $x, y$. Hence, one way to infer whether a single operation is linear is to simply take $x, y$ to be the inputs to the operation $f$ in the original computation graph as produced by two random batches of inputs to the full graph, and check whether the above equation holds for scalar $a, b$. However, as specifying whether an operation is linear or nonlinear is relatively simple and does not impose undue burden upon the programmer, we chose to manually specify, for each operation, whether it is linear to avoid additional inference overheads at synthesis time.

\paragraph{Abstract Interpretation}

Our programs are represented as a directed acyclic computation graph, flattened via topological sort. Hence, given a mapping of primitive operations to their linear property, we simply scan all the nodes in topological order to compute the depth property for every pair of inputs and outputs in a single pass.

\subsection{Shape Property}

The shape property specifies the shape of every output of the subgraph, given a shape for every input of the subgraph. More explicitly, for any program $p$, the abstract interpretation $p_S^\alpha$ maps an input shape $u = (i_1, \ldots, i_n)$ to an output shape $v = (o_1, \ldots, o_m)$, where given any tensor $I$ with shape $u$, the output tensor $O = p(u)$ has shape $v$. We use JAX’s built-in shape inference to easily infer the shape property without executing the graph on any realized input tensors. 

\subsection{Example}

\cref{fig:example:inference} illustrates how we infer program properties step-by-step, using the randomly selected subgraph from the running example in \cref{fig:overview}. The program properties are initialized according to the identity program (i.e., the mixing property is an identity matrix, the depth property is zero, and the shape property preserves the input shape). We then interpret each operation on its input properties abstractly, as described in this section, to produce its output properties. The properties of the entire subgraph are the final outputs from the abstract interpretation.

\begin{figure}[t]
    \centering
    \includegraphics[width=\linewidth]{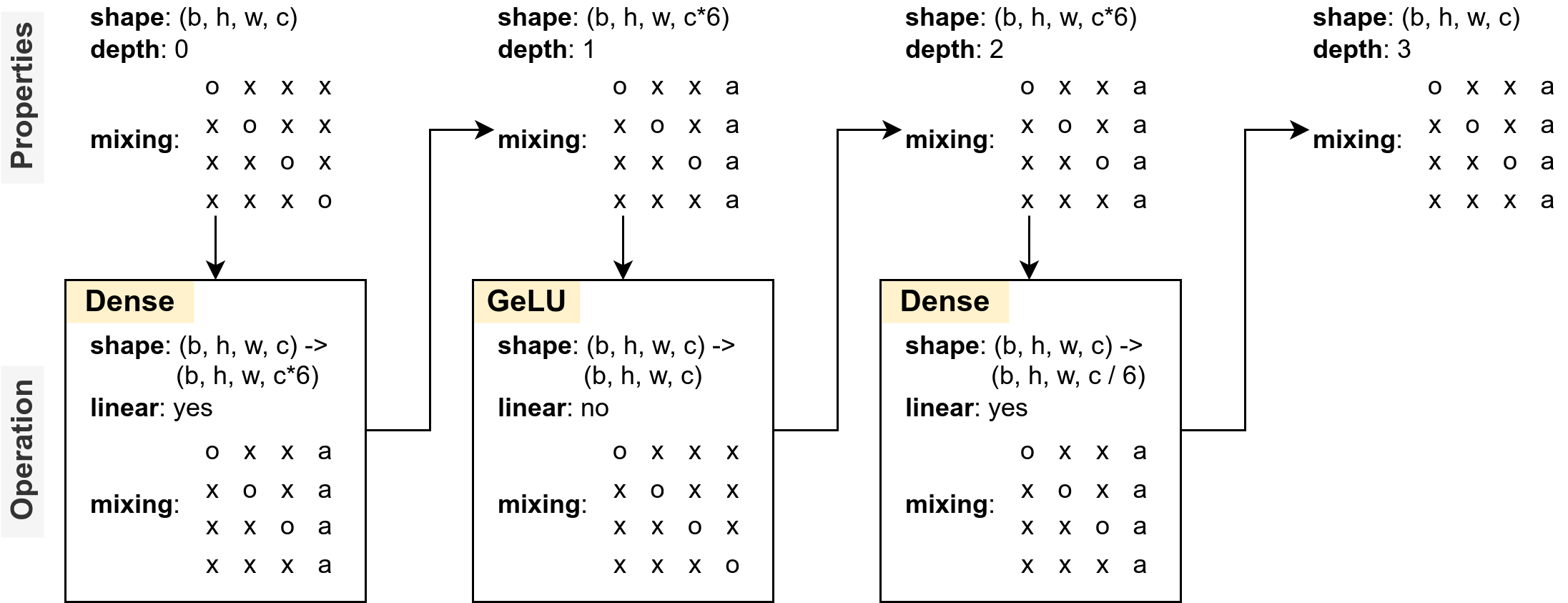}
    \caption{The abstract interpretation to infer program properties of the randomly selected subgraph in \cref{fig:overview}}
    \label{fig:example:inference}
\end{figure}

\subsection{Discussion}

We conclude by discussing how our program properties relate to the motivation provided at the beginning of this section. Our properties are designed to characterize neural architectures (and mutations thereof) as mathematical functions, the intuition being that deep neural networks that share similarities as mathematical functions are likely to have similar performances as well. In terms of the two semantic desiderata: each point in the abstract space contains points that are similarly expressive as mathematical functions, and hence more likely to have similar performance (in terms of FLOPS, parameter count, and accuracy); similarly, nearby points in the abstract space yield similar subgraphs from a functional perspective.

We can further illustrate this point by considering the examples of degenerate architectures from the beginning of this section. The depth property guides the search away from the first two degenerate architectures: (1) a sequence of activation functions, and (2) a sequence of linear functions, while the mixing property rules out the last degenerate architecture: (3) a sequence of pooling and activation functions. The shape property ensures that the synthesized subgraph results in a legal substitution, so that the resulting graph still represents a well-defined, executable computation. In particular, if the output shape is too small, the resulting computation graph may not be well-defined due to a subsequent operation producing a dimension of size zero; if the output shape is too large, the remainder of the computation graph may not be executable due to out of memory errors. To summarize, any architecture which satisfies the same mixing, depth, and shape properties as a deep neural network should at least also qualify as a valid deep neural network in the sense of \cref{def:mlp}, allowing the evolutionary search to prune large portions of the concrete search space consisting of legal programs that represent degenerate architectures.

\section{The Synthesis Algorithm}
\label{sec:synthesis}

This section presents our synthesis algorithm formally and shows how we leverage a notion of distance between programs and properties to achieve progress, soundness, and completeness. To generate a program of length $n$ given a search space $T$ of transformations, the complexity of our algorithm scales \emph{linearly} with $n$ and $|T|$. In contrast, the na\"ive enumerative synthesis algorithm that tries all sequences of transformations until finding a satisfying sequence has complexity $O(|T|^n)$.

\subsection{Distance Functions and Covering Sets}
\label{sec:distance_covering}

\begin{definition}
\label{def:distance}
Given a program property $\Pi = (\mathcal{V}, \le, \alpha)$, a \textbf{distance function} is a function $d : \mathcal{P} \times \mathcal{V} \mapsto \mathbb{R}^+ \cup \{\infty\}$ such that:
\begin{enumerate}
    \item $d(p, v) \ge 0$ for all $p \in \mathcal{P}, v \in \mathcal{V}$. 
    \item $d(p, v) = 0 \iff p \models v$.
    \item $d(p, v) = \infty$ $\iff$ there does not exist any finite sequence of transformations $t = t_1 \circ t_2 \circ \cdots \circ t_n$ such that $t(p) \models v$.
\end{enumerate}
\end{definition}

For the distance function to be useful for synthesis, it must be accompanied by a set of transformations which can be used to minimize it:

\begin{definition}
\label{def:covering}
Given a space of programs $\mathcal{P}$, a program property $\Pi = (\mathcal{V}, \le, \alpha)$, and a distance function $d$, a set of transformations $T \subseteq \mathcal{T}$ is a \textbf{covering} if for all $p \in \mathcal{P}$ and $v \in \mathcal{V}$ with $d(p, v) > 0$, either (1) there exists $t \in T$ such that $d(t(p), v) < d(p, v)$ or (2) $d(p, v) = \infty$.
\end{definition}

For all feasible synthesis tasks, the covering $T$ is guaranteed to be able to make progress as measured by the distance function $d$.
However, \emph{efficient} synthesis requires a stronger condition:

\begin{definition}
\label{def:uniform-covering}
A covering $T$ is a \textbf{uniform covering} if there exists a constant $\epsilon > 0$ such that for all programs $p \in \mathcal{P}$ and properties $v \in \mathcal{V}$ with $d(p, v) > 0$, there exists $t \in T$ such that 
\begin{align}
    d(t(p), v) + \epsilon \le d(p, v)
\end{align}
\end{definition}

In other words, the amount of progress can be bounded \emph{uniformly} from below (i.e., the same bound applies to all programs and properties). A uniform lower bound guarantees that the distance will be zero after a finite number of steps, rather than approach zero in the limit. Additionally, note that a necessary condition for the viability of a distance is that the full space of transformations $\mathcal{T}$ should be a covering. However, as our algorithm will scale linearly with $|T|$, we would also like $T$ to be as small as possible, whereas $|\mathcal{T}|$ is generally infinite (as in our case). 

\subsection{Progressive Synthesis}
\label{sec:prog_syn}

In this section, we will assume black-box access to a distance function $d$ for a program property $\Pi = (\mathcal{V}, \le, \alpha)$, as well as a uniform covering $T \subseteq \mathcal{T}$. Given an initial program $p_0$ and a program property value $v$, our synthesis procedure iteratively produces programs $p_1, p_2, \ldots$ until it finds a program $p^* \models v$, where we use $d$ to help in selecting a transformation at each step.

\cref{alg:synthesis} describes the \textbf{greedy progressive synthesis} algorithm that synthesizes a program $p^*$ satisfying any (feasible) program property $v$. Each iteration $i$ starts with a partial program $p_{i-1}$, where $p_0$ is initialized to the input program. Given the distance function $d$, for each transformation $t$ in the uniform covering $T$, the algorithm computes the distance from the transformed program $t(p_{i-1})$ to the desired property $v$, and selects the transformation $t_i$ which achieves the minimum distance. This process is repeated until the distance is 0, at which point the algorithm terminates and returns the sequence of transformations.

\begin{algorithm}[t]
\small
\caption{Greedy progressive synthesis}
\label{alg:synthesis}
\begin{algorithmic}[1]
\Require initial program $p$, program property $v$, distance function $d$, uniform covering $T$
\Ensure $\not\models$ if $(p, v)$ is infeasible; otherwise, a sequence of transformations $\tau$ such that $\tau(p) \models v$
  \If{$d(p,v) = \infty$}
    \State \Return $\not\models$
  \EndIf
  \State $\tau \gets \{ \}$, $p_0 \gets p$
  \For{$i$ = 1, 2, \ldots}
    \If{$d(p_{i-1},v) = 0$}
        \State \Return $\tau$
    \EndIf
    \State $t_i \gets \argmin_{t \in T} d(t(p_{i-1}), v)$
    \State $p_i \gets t_i(p_{i-1})$
    \State $\tau \gets (t_1, \ldots, t_i)$
  \EndFor
\end{algorithmic}
\end{algorithm}

The following theorem characterizes the correctness and efficiency of progressive synthesis. 

\begin{theorem}
\label{thm:prog_syn}
Given a distance function $d$ and a uniform covering $T$, the progressive synthesis algorithm satisfies the following three properties:
\begin{enumerate}
    \item Soundness: if the algorithm returns $p^*$, then $\alpha(p^*) \models v$.
    \item Completeness: if the algorithm returns $\not\models$, then the synthesis task is infeasible.
    \item Progress: the algorithm terminates after a finite amount of time, which is linear in the length of the final sequence of transformations and linear in $|T|$.
\end{enumerate}
\end{theorem}

We defer the proof to the \cref{appendix:proofs:prog_syn}.

\subsubsection{Stochastic Progressive Synthesis}
\label{sec:stoch_prog_syn}

Selecting a random feasible operation (with probability $p > 0$) and an operation that makes progress (with probability $1-p$) increases the expected length of the synthesized subgraph, but ensures that any satisfying subgraph has a positive probability of being synthesized. We use this variant rather than a pure greedy approach in our implementation to encourage greater diversity.

\subsubsection{Progressive Synthesis as a Universal Synthesis System}
\label{sec:prog_syn_universal}

A natural question is the extent to which it is necessary to approach the task of program synthesis by defining a distance function with a covering. Alternatively, we would like to know whether there are synthesis tasks which have efficient solutions, but cannot be solved efficiently via progressive synthesis.

The main result of this section identifies a natural class of synthesis algorithms that are exactly as powerful as progressive synthesis. In particular, we define the class of \emph{recursively consistent} synthesis algorithms as the set of algorithms $A$ that can be efficiently ``restarted'' from a partial solution, e.g., if $A(p_0, v) = (t_1, t_2, \ldots, t_n)$, then $A(t_1(p_0), v) = (t_2, \ldots, t_n)$.

The following (informal) theorem states that any task that can be efficiently solved by a recursively consistent algorithm can also be solved efficiently by progressive synthesis, and vice-versa.

\begin{theorem}[Informal Version.]
\label{thm:synthesis_complexity_informal}
Given a synthesis task specified by a space of programs $\mathcal{P}$, a program property $\Pi = (\mathcal{V}, \le, \alpha)$, and a set of primitive transformations $E$, the following are equivalent:
\begin{enumerate}
    \item There exists a sound, complete, and recursively consistent synthesis algorithm $A$ that runs in time polynomial in the length of the synthesized transformations $|A(p, v)|$, for all $p \in \mathcal{P}$, $v \in \mathcal{V}$.
    \item $E$ is a uniform covering with respect to an efficiently computable distance $d$.
\end{enumerate}
\end{theorem}

The second statement guarantees that progressive synthesis is sound, complete, and runs in polynomial time. In particular, our analysis takes into account the complexity of the distance function $d$ (cf. \cref{thm:prog_syn}, which treats $d$ as an oracle). We state and prove the rigorous version of this result in \cref{appendix:proofs:prog_syn}.

\subsection{Multiple Properties from Monotonic Transformations}
\label{sec:multiple_properties}

This section discusses how to adapt our progressive synthesis algorithm when there are multiple properties of interest. For instance, assume that we have two program properties $\Pi_1$ and $\Pi_2$. Each property $\Pi_i$ has a distance function $d_i$ with a uniform covering $T_i$. The objective is to synthesize a sequence of transformations satisfying a set of program property values $S = \{v_1 \in \mathcal{V}_1, v_2 \in \mathcal{V}_2\}$.

The main problem is that applying a transformation $t \in T_1$ that decreases the distance $d_1$ to the first property $v_1$ may \textit{increase} the distance $d_2$ to the second property $v_2$. In other words, we cannot guarantee there exists a transformation $t \in T_1 \cup T_2$ that simultaneously improves both $d_1$ and $d_2$. In the worst case, progress is impossible: we may trade-off between $d_1$ and $d_2$ indefinitely.
To resolve this issue, we introduce a notion of monotonic transformations:

\begin{definition}
A transformation $t \in \mathcal{T}$ is \textbf{monotonic} with respect to a property $\Pi$ and distance $d$ if for all programs $p \in \mathcal{P}$ and properties $v \in \mathcal{V}$, $d(t(p), v) \le d(p, v)$.
\end{definition}

A monotonic transformation, by definition, cannot cause the property $\Pi$ to regress. We will also refer to a set $T \subseteq \mathcal{T}$ as monotonic if all $t \in T$ are monotonic. For instance, the set of all transformations $\mathcal{T}$ is monotonic with respect to the depth property.

The following theorem gives a sufficient condition for progressive synthesis to succeed over multiple properties. We defer the proof to \cref{appendix:proofs:multiple_properties}.

\begin{theorem}
\label{thm:prog_syn_multi}
Let $\mathbb{P} = \{\Pi_i = (\mathcal{V}_i, \le_i, \alpha_i)\}_{i = 1}^N$ be a set of program properties, each with a distance function $d_i$ and covering $T_i$, respectively, and let $S = \{v_i \in \mathcal{V}_i\}_{i = 1}^N$ be a set of program property values. If for all $i$, $\cup_{j \ne i} T_j$ is monotonic with respect to $d_i$, then $d(p, S) = \sum_i d_i(p, v_i)$ is a distance function for $\mathbb{P}$ with covering $T = \cup_{i=1}^N T_i$. Furthermore if each covering $T_i$ is uniform, then so is $T$.
\end{theorem}

\subsection{Progressive Synthesis for NAS}
\label{sec:prog_syn_nas}

This section shows how to apply the progressive synthesis algorithm to our setting of synthesizing subgraphs satisfying the program properties described in \cref{sec:properties}. As a technical matter, we need to first interpret the definition of distance functions and covering sets from \cref{sec:distance_covering} in the context of an abstract interpretation $\alpha$. Briefly, given an \emph{abstract distance function} $d^\alpha(u, v)$ over properties $u, v \in \mathcal{V}$ (where $d^\alpha(u, v) = 0$ implies $u \ge v$), we say $T \subseteq \mathcal{T}$ is an \emph{abstract uniform covering} if there exists $\epsilon > 0$ such for all $u, v \in \mathcal{V}$ with $d^\alpha(u, v) > 0$, there exists $t \in T$ such that
\begin{align}
    d^\alpha(t^\alpha(u), v) + \epsilon \le d^\alpha(u, v)
\end{align}
Hence, if $\alpha(p_0) = u$ and $t_1, \ldots, t_n$ is a sequence of transformations such that the distance of the abstractly interpreted property is zero, i.e.,
\begin{align}
    d^\alpha((t_n^\alpha \circ \cdots \circ t_1^\alpha)(u), v) = 0
\end{align}
then $p = (t_n \circ \cdots \circ t_1)(p_0) \models v$. A complete treatment is provided in \cref{appendix:abstract_prog_syn}.

In order to apply the version of Theorem \ref{thm:prog_syn} using abstract interpretations to our setting, in the following sections we will prove:
\begin{enumerate}
    \item For each of the mixing, depth, and shape properties, there exists an abstract distance $d^\alpha_i$ computable in polynomial time with a uniform abstract covering of transformations $T_i \subset \mathcal{T}$.
    \item $T = \cup_i T_i$ is monotonic with respect to both the depth and mixing properties.
    \item $T \cap M_S$, where $M_S$ is the set of transformations that preserve the shape of the output tensor, is a uniform covering for both the depth and mixing properties. 
\end{enumerate}

For every property, we take $T_i = E$, the set of transformations $E$ consisting of adding a single primitive operation. \cref{appendix:primitives} describes the full set of primitive operations we support. Then $T = E$, and the following result is immediate:

\begin{theorem}
\label{thm:progressive_nas}
The progressive synthesis algorithm for \ourname{} is sound, complete, and runs in time linear in $|p^*|$ and $|E|$, where $p^*$ is the synthesized subgraph, and $E$ is the set of primitive operations.
\end{theorem}

\subsubsection{Sequential Subgraphs Without Reshaping}
\label{sec:nas_properties}

As a first step, we will focus on synthesizing subgraphs that are sequential (i.e., each operation consumes exactly one input and each output is consumed exactly once) and preserve the dimensions of their inputs. In other words, we restrict our attention to the set of \emph{simple primitives} $E_s \subseteq E$ that take one input and produce one output, and also do not reshape or transpose their inputs. Synthesis occurs over the corresponding space of \emph{simple transformations} $\mathcal{T}_s = E_s^*$. The set of program properties values $\mathcal{V}$ is limited to those satisfiable by a sequential subgraph with operations from $E_s$. We present the key results below; all proofs are deferred to \cref{appendix:proofs:prog_syn_nas}.

\subsubsection{Mixing Property}

We denote the mixing property as $\Pi_M = (\mathcal{V}_M, \le, \alpha_M)$. Without any reshape operations, the number of input and output dimensions are equal; hence, the space of property values contains only \emph{square} matrices with entries in \{\la{}, \lm{}, \lo{}, \lx{}\} satisfiable by a sequential subgraph (cf. \cref{fig:mixing}). 
We define the following abstract distance function for the mixing property.

\begin{definition}
Define $d^\alpha_M(U, V)$, where $U, V \in \mathcal{V}_M$, to be the number of entries $i,j$ such that $V_{ij} > U_{ij}$. Then $d^\alpha_M$ is a distance on the mixing property.
\end{definition}

The following lemma establishes that the set of primitive ops $E_s$ is monotonic with respect to the mixing property.

\begin{lemma}
\label{lemma:preserve}
$\forall t_1, t_2 \in \mathcal{T}_s$ and $U \in \mathcal{V}_M$, $(t^\alpha_2 \circ t^\alpha_1) (U) \ge \max(t^\alpha_2 (U), t^\alpha_1 (U))$.
\end{lemma}

\begin{theorem}
\label{thm:mixing_monotone}
The set of simple primitives $E_s$ is a monotone uniform abstract covering for the mixing property with respect to $d^\alpha_M$.
\end{theorem}

\subsubsection{Depth Property}

For the depth property $\Pi_D = (\mathcal{V}_D, \le, \alpha_D)$, the space $\mathcal{V}_D$ of depth properties for sequential subgraphs consists of all the nonnegative integers. The function $d^\alpha_D(u, v) = \max(0, v - u)$ is an abstract distance for this property. It is easy to see that the set of simple primitive operations forms a uniform abstract covering with respect to this distance, since we can always append alternating linear and nonlinear operations to decrease the distance to $v$. The depth property is also trivially monotonic for all transformations $\mathcal{T}$.

\begin{theorem}
\label{thm:depth_monotone}
The set of simple primitives $E_s$ is a monotone uniform abstract covering for the depth property with respect to $d^\alpha_D$.
\end{theorem}

\subsubsection{Shape Property}

In general, there are two ways a (dimension-preserving) operation can change the shape of a tensor. The first is by changing the channel dimension to an arbitrary integer (e.g., a dense layer). The second is by downsampling the spatial dimensions by an integer factor (e.g., a \mbox{2~$\times$~2} pooling operation would decrease the spatial dimensions by a factor of 1/2). Convolutional layers have the potential to perform both at the same time. Note that we do not support operations that change the spatial dimensions by other amounts (for instance, strided convolutions without padding), which does not materially limit our search space; in fact, all of the base architectures considered in \cref{sec:case_studies} respect this condition.

The space of all shape properties $\mathcal{V}^m_S$ of dimension $m \ge 1$ consists of all $m$-tuples of positive integers, which specify the shape of an $m$-dimensional tensor: the first entry is the batch dimension (which is never changed); the last entry is the channel dimension; and all intermediate entries are the spatial dimensions. 

\begin{definition}
The following function $d^\alpha_S(A, B)$ is an abstract distance on the shape property:
\begin{align}%
    d^\alpha_{channel}(A, B) &:= 
    \begin{cases}
    1,& \text{if } a_{channel} \,\ne \, b_{channel}\\
    0& \text{otherwise}
    \end{cases} \\
    d^\alpha_{spatial}(A, B) &:= 
    \begin{cases}
    \sum_{i=1}^{m-2} a_i / b_i,& \text{if } b_i \text{ divides } a_i \quad \forall i\\
    \infty& \text{otherwise}
    \end{cases} \\
    d^\alpha_S(A, B) &:= d^\alpha_{channel}(A, B) + d^\alpha_{spatial}(A, B)
\end{align}
where $A = (batch, a_1, \ldots, a_{m-2}, a_{channel}) \in \mathcal{V}^m_S$ and $B = (batch, b_1, \ldots, b_{m-2}, b_{channel}) \in \mathcal{V}^m_S$
\end{definition}

Next we show that $E_s$ is a uniform covering. Indeed, it is always possible to decrease $d^\alpha_S$ by appending one of two simple primitives: if any spatial dimensions are not equal, a pooling operator $p$ with an appropriately chosen window yields: 
\begin{align}
    d^\alpha_{spatial}(e_{pool, S}^\alpha(A), B) &< d^\alpha_{spatial}(A, B) \\
    d^\alpha_{channel}(e_{pool, S}^\alpha(A), B) &= d^\alpha_{channel}(A, B)
\end{align}
If $d_{channel} > 0$ then, a dense layer $e_{dense}$ achieves:
\begin{align}
    d^\alpha_{spatial}(e_{dense, S}^\alpha(A), B) &= d^\alpha_{spatial}(A, B) \\
    0 = d^\alpha_{channel}(e_{dense, S}^\alpha(A), B) &< d^\alpha_{channel}(A, B)
\end{align}

\begin{theorem}
\label{thm:shape_uniform}
The set of simple primitives $E_s$ is a uniform abstract covering for the shape property with respect to $d^\alpha_S$.
\end{theorem}

\subsubsection{Example}

Consider the running example in \cref{fig:overview}. The goal is to synthesize a subgraph that satisfies the mutated properties. \cref{fig:example:synthesis} illustrates how our synthesis algorithm makes progress toward this objective at each step by appending an operation that reduces the abstract distance to the target properties (in red). The first operation BatchNorm makes progress on the depth properties. In the second step, the nonlinear activation function SiLU changes neither the mixing nor shape properties, but increases depth. This process repeats until the synthesized subgraph reaches a total distance of 0.

\begin{figure}[t]
    \centering
    \includegraphics[width=\linewidth]{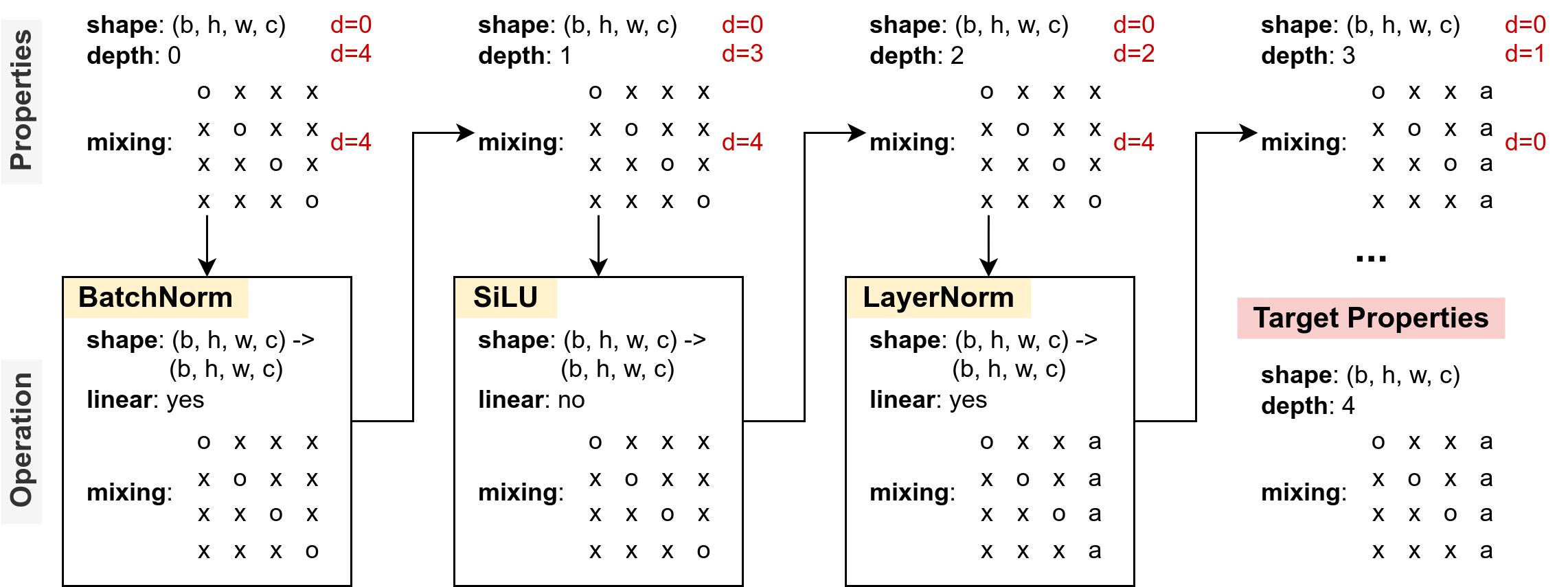}
    \caption{Illustration of the synthesis of the first three operations in the subgraph from \cref{fig:overview}. The progressive synthesizer adds an operation that reduces the distance to the target properties in each step. The annotation ``d='' on the right of each property indicates the distance to the target property.} 
    \label{fig:example:synthesis}
\end{figure}

\subsection{Compressing the Search Space}

Many transformations are indistinguishable with respect to the abstract properties. For instance, in our case, a \mbox{2~$\times$~2} average pool and a \mbox{3~$\times$~3} max pool have the same mixing, depth, and shape properties. Hence, during synthesis it is sufficient to measure the progress of a single pooling operator, which is representative of the progress for any other pooling operator. Formally, given two transformations $s$ and $t$, if their abstract interpretations are equivalent, i.e., $s^\alpha = t^\alpha$, then we can safely remove one of them from the covering set $T$ while preserving the safety and performance of the progressive synthesis algorithm.
Our final synthesis algorithm considers only 87 out of the  1,402 total primitive transformations during synthesis. The remaining 1,315 operations consist largely of different settings for spatial kernels (e.g., pooling or convolutions) and feature groups (e.g., group normalization or grouped convolutions).

After a satisfying subgraph has been synthesized, we perform a second pass to randomly replace each operation $e$ in the synthesized subgraph with an operation from the set originally represented by $e$; by construction, the replacements preserve the desired properties of the subgraph. For example, we can safely change the momentum of the first BatchNorm synthesized in \cref{fig:example:synthesis}. To ensure the final subgraph is balanced, we only use operations of the same type as representatives (e.g., both SiLU and ReLU are in the 87 operations, even though they have equivalent abstract interpretations).

\subsection{Extensions Beyond Sequential Subgraphs}
\label{sec:multi_in_out_syn}

This section briefly describes how to extend our synthesis algorithm beyond sequential subgraphs composed of simple operations. Supporting mutations of subgraphs with multiple inputs and outputs proceeds in three steps. First, we identify the reshape and binary operations in the subgraph, and decompose the subgraph into a set of sequential subgraphs composed of simple operations connected by the reshape and binary operations. Second, we can mutate the structure of the reshape and binary operators. Finally, we synthesize a replacement for each of the sequential subgraphs using our aforementioned techniques.

The details of the structure mutation are deferred to \cref{appendix:structure_mutation}. The key is ensuring that the mutated structure remains feasible (e.g., the resulting structure should still be connected). In practice, most common architectures consist of a sequential backbone, augmented by basic residual connections. Furthermore, reshaping occurs relatively infrequently and is usually only applied at the end of the network to flatten the features for extracting logits. Hence, our implementation does not perform the second step (structure mutation).

\section{Implementation}

\paragraph{Evolutionary search.} We initialize the evolutionary search with a given base architecture (e.g., ResNet-50). At each iteration, we randomly select a single individual to mutate; individuals with higher fitness are more likely to be chosen. This individual is trained from scratch and its fitness is evaluated. This procedure is repeated for a fixed number of iterations, at which point we terminate the search and return the population of evolved individuals. For more details on our evolutionary search procedure with multiple objectives, see \cref{appendix:evol}.

\paragraph{Mutating Architectures.} We perform our architecture search at the level of blocks. Each base architecture is split into predefined blocks, where each block consists of mutable components, possibly wrapped within a immutable residual connection. For instance, we have the following block types in the ResNet-50 architecture: (1) bottleneck block with a \mbox{2~$\times$~2} spatial down-sampling and (2) bottleneck block without any spatial down-sampling. We decompose the Vision Transformer as follows: (1) self attention block, and (2) MLP block. We also use a small 2-block CNN without residual connections; each block consists of a convolution that doubles the feature depth, a ReLU activation, and an Average Pool that reduces the spatial resolution by half. Hence, all our models consist of a sequential stack of blocks with some residual connections.

\paragraph{Mutating Blocks.} To mutate an architecture, we select a block at random, then select a subgraph within that block at random. We then synthesize a replacement subgraph with mutated properties and replace the subgraph into the block. Each mutation is also applied at random to other blocks of the same type within the model. In addition to mutating individual blocks, we also include block deletion and block duplication in our search space. When duplicating a block, we select a block at random and insert a copy into another random location in the model.

\paragraph{Mutating properties.} In general, we mutate properties in the direction of relaxing constraints: since a synthesized program $p$ satisfies a property value $v$ if $\alpha(p) \ge v$, without any mutations the synthesis is bound to only return subgraphs (and hence architectures) with the same property values or greater. Relaxing properties also ensures that the synthesis task is feasible, since the original subgraph automatically satisfies the mutated properties. The only exception is the depth property, which is preserved with 50\% probability, otherwise we mutate it up or down by up to 2 uniformly at random. We mutate the shape property by removing the output shape requirement with 50\% probability. We mutate the mixing property by removing pairings with 50\% probability.

\paragraph{Synthesis.} We implement stochastic progressive synthesis as described in \cref{sec:stoch_prog_syn,sec:multi_in_out_syn}. At each step, an operation is selected at random with probability proportional to $1 / (1 + d)$, where $d$ is the distance to the desired properties after appending the operation. Once the synthesized subgraph reaches the size of the original subgraph, we switch to the greedy algorithm, which selects an operation that achieves the minimum distance. If synthesis does not complete after two additional steps, we consider synthesis to have failed and report no individual for that generation.

\paragraph{Environment.} We use JAX \citep{jax2018github} to train and evaluate architectures on TPU v2 (for CIFAR-10) and v3 (for ImageNet) accelerators. Our implementation is $\sim$10,000 lines in Python.

\section{Experimental Results}
\label{sec:results}

This section presents results from experiments that evaluate the ability of our approach to generate interesting novel architectures (\cref{sec:case_studies}), assess the importance of key components in our approach (\cref{sec:ablation}), and compare against prior approaches (\cref{sec:existing_approches}). In each experiment, we initialize the evolutionary search with a base model (displayed as yellow stars in the plots\footnote{We evaluate the base model twice to reduce the noise from the stochasticity of the training procedure.}). We run the search for a fixed number of trials; each trial involves training a proposed candidate for a fixed number of epochs to evaluate its accuracy on a given dataset. The search trades off the primary objective of accuracy (higher is better) against some secondary objectives: FLOPS (lower is better), number of parameters (lower is better), and training throughput (images per second, higher is better). \cref{tab:experiment_setup} in \cref{appendix:setup} describes the setup for each experiment in more detail.

\subsection{Case Studies}
\label{sec:case_studies}

\subsubsection{CIFAR-10}
\label{sec:cifar10}

Our first set of experiments focuses on the CIFAR-10 dataset with 50,000 training images and 10 classes. We are specifically interested in how rapidly our evolutionary search progresses when seeded with models that leave ample room for improvement. Our first experiment seeds a small 2-layer CNN with residual connections, which is suboptimal in terms of accuracy (80\%) but otherwise has good performance characteristics (low FLOPS, parameter count). Our method discovers a model with substantially increased accuracy, from 81.3\% to 89.6\%, while also decreasing the number of parameters slightly.

We also run a second search seeded with the ResNet-34 architecture. This base model has much higher accuracy on CIFAR-10 (91.5\%) but is less efficient. Figure \ref{fig:baseline_resnet34} displays the results. Keeping accuracy constant, our search discovers one model that decreases the FLOPS by 68.7\%, and a second model that reduces parameters by more than 96\% (from 21.3 to 0.8 million parameters). A third model increases the accuracy to 93.0\% using only 5.8 million parameters, a 73.1\% reduction. These results support the hypothesis that our approach is able to make rapid progress and discover substantially improved models on simpler tasks.

\begin{figure*}[b]
    \centering
    \includegraphics[width=.8\linewidth,trim={0 0 0 2cm},clip]{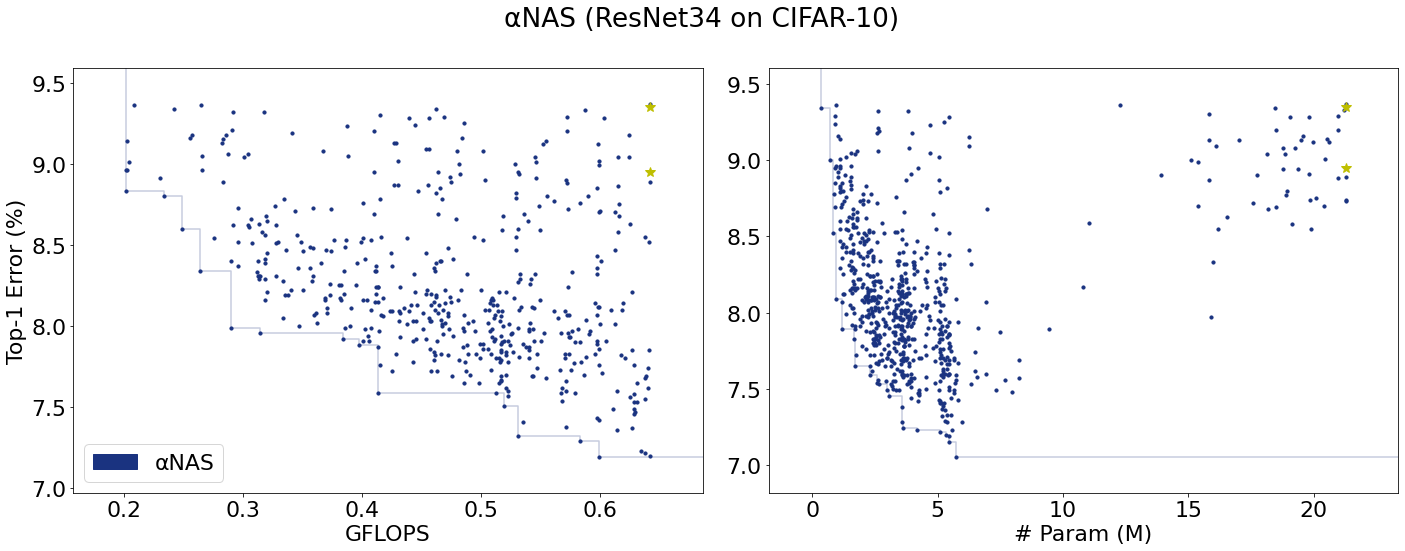}
    \caption{Evolving a ResNet-34 model on CIFAR-10 using \ourname{}. Yellow stars represent the base architecture.} 
    \label{fig:baseline_resnet34}
\end{figure*}

\subsubsection{ImageNet}

We next present results using the ILSVRC 2012 version of the ImageNet dataset with 1.2 million training images and 1,000 classes. Our objective is to test whether our approach can generate improvements in a more difficult setting.

\begin{figure}[t]
    \centering
    \includegraphics[width=.8\linewidth,trim={0 0 0 2cm},clip]{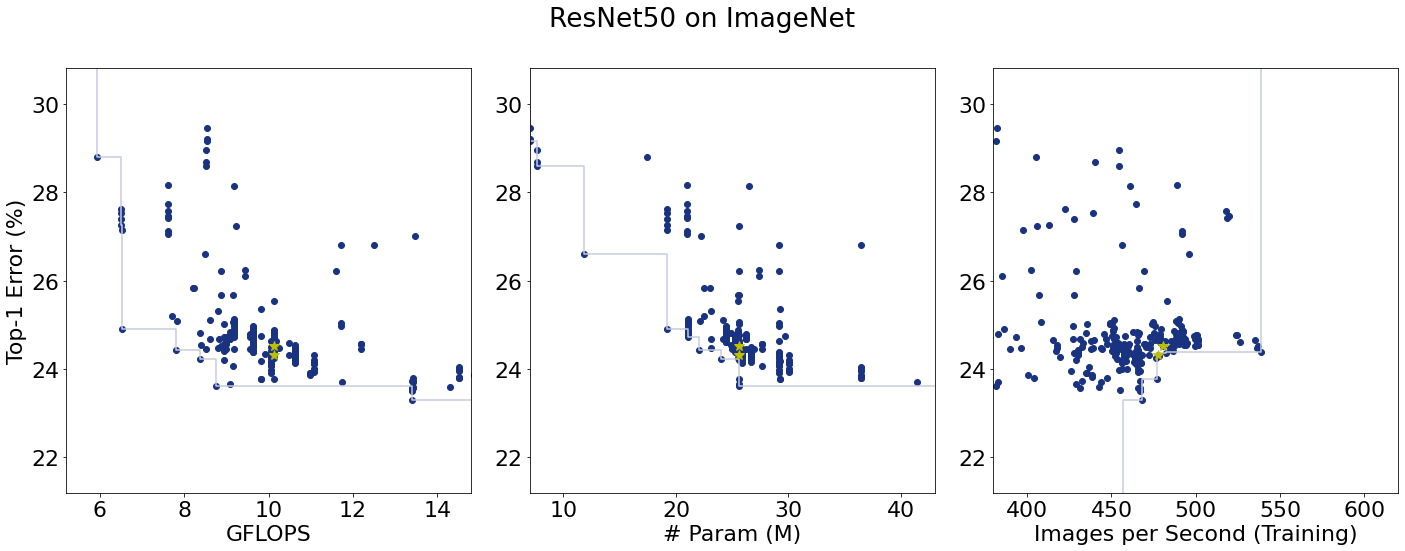}
    \caption{Evolving a ResNet-50 model on ImageNet. We optimize for increased accuracy, decreased FLOPS, increased training speed. For the same accuracy, we have decreased GFLOPS by 23\% and increased the training speed by 12\%.} 
    \label{fig:resnet50_imagenet}
\end{figure}

\begin{figure}[t]
    \centering
    \includegraphics[width=.8\linewidth,trim={0 0 0 2cm},clip]{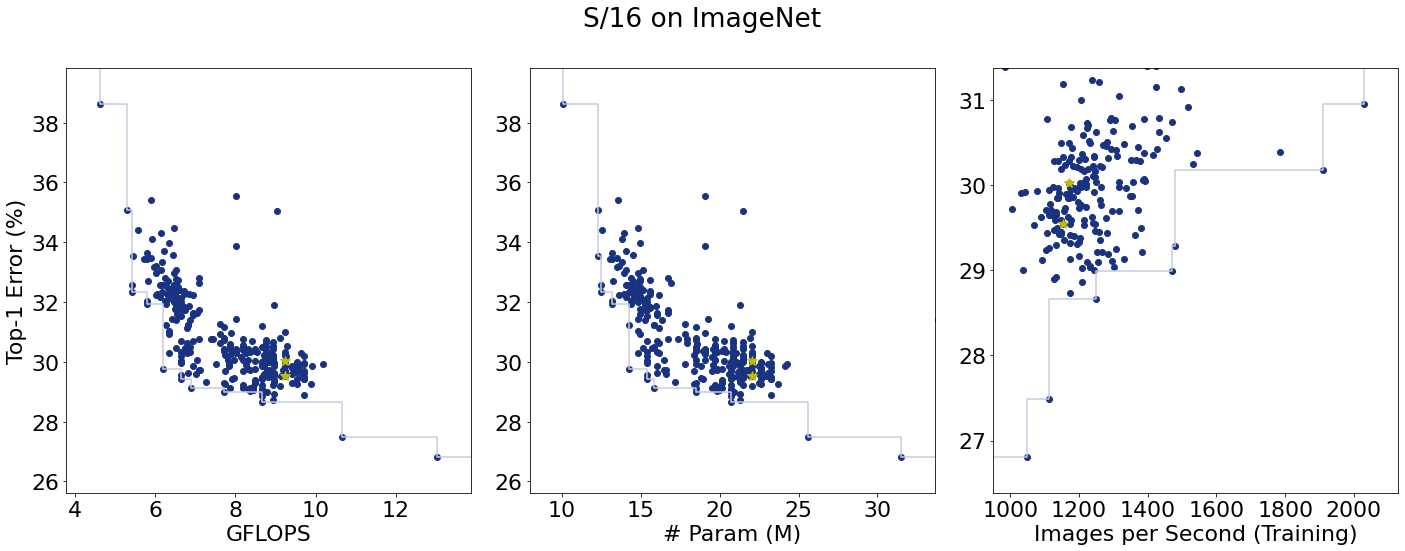}
    \caption{Evolving a ViT-S/16 model on ImageNet. We optimize for increased accuracy, decreased FLOPS, and increased training speed. Our search returns a model with the same accuracy using 30\% fewer FLOPS and parameters.} 
    \label{fig:s16_imagenet}
\end{figure}

\begin{figure*}[t]
    \centering
    \includegraphics[width=.8\linewidth]{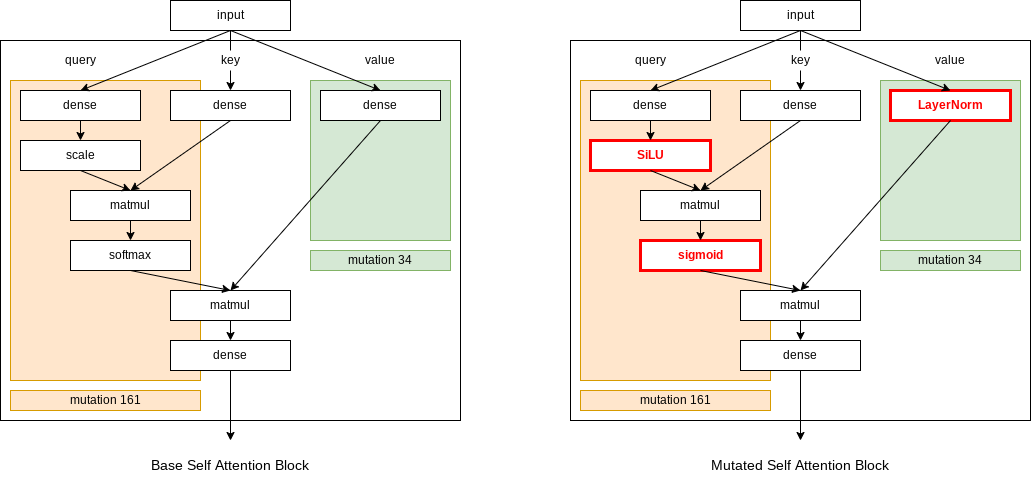}
    \caption{Two mutations of the self-attention block in the ViT architecture. The final model is a strict improvement over the base architecture in all objectives.} 
    \label{fig:vit_mutation}
\end{figure*}

\paragraph{ResNet-50}
The first experiment is seeded with the standard ResNet-50 architecture \citep{he2015deep}. The ResNet architecture has seen many variants since being introduced in 2015, and its main innovation (the residual connection) remains a central feature of many architectures today. It remains a common benchmark for many vision tasks.
Figure \ref{fig:resnet50_imagenet} displays our results. After evaluating 400 individuals, \ourname{} discovers a model that has the same accuracy but 23\% fewer FLOPS (from 10.1 to 7.8 GFLOPS) and 13.7\% fewer parameters (from 25.6 to 22.0 million); and a second model which increases the training speed 12\% (from 480 to 538 images per second per core) with slightly better accuracy. We also discover several models with nearly 80\% accuracy (a 4.5\% increase, not shown in the plot), though at a nearly 15-fold increase in both FLOPS and parameter count.

\paragraph{ViT-S/16}

The second experiment is seeded with the S/16 variant of the Vision Transformer architecture (ViT) \citep{dosovitskiy2021image}. ViT is a relatively new architecture, and many variants have been proposed recently to improve the base architecture in a number of ways, such as accuracy, training speed, inference speed, and sample efficiency. We are interested in seeing whether \ourname{} can automatically discover similar improvements.

Figure \ref{fig:s16_imagenet} plots our results. After evolving a total of 400 individuals, \ourname{} discovers a model as the 278th individual that decreases FLOPS by 28\% and parameter count by 30\% (with a negligible increase in accuracy). The mutation of the MLP block in \cref{fig:overview} is present in this model. The resulting block has almost no trainable parameters and very few FLOPS. Similar mutations of the MLP block are common throughout the population.
Such a mutation is nearly impossible to perform using single-operation mutations, since it requires eight mutations and simultaneously increasing depth while replacing linear operations with cheaper ones. If a linear operator is replaced before the depth is increased, the accuracy could suffer; but if depth is increased before replacing the linear operator, the parameter count and FLOPS will increase. In both cases, the intermediate model is worse than the parent, making it unlikely to be selected for further mutation.

We present another series of mutations in \cref{fig:vit_mutation} that increases accuracy by nearly 1\% and improves FLOPS, parameter count, and training speed by 6--7\%. Mutation 34 appears in nearly 20\% of the evolved population, as it strictly improves every metric for the base architecture.
Mutation 161 replaces the multiplicative scaling with a self-gating SiLU unit, and also replaces the Softmax with a cheaper but related Sigmoid. These analyses suggest that \ourname{} is able to discover a wide range of novel architectural improvements.

\paragraph{EfficientNet-B0}

Our final experiment is seeded with the EfficientNetV1-B0 architecture \citep{efficientnet}. EfficientNet is a family of extremely parameter- and FLOP- efficient models, and presents a strong baseline produced by existing NAS techniques. We are interested in evaluating whether our methods can deliver an improvement over a state-of-the-art model. The results after training each candidate for 90 epochs are displayed in \cref{fig:efficientnet_imagenet}. Although EfficientNet is indeed a more difficult architecture to improve on, \ourname{} is still able to discover a model (named E') which improves both FLOPS and parameter count by 7--8\% with similar accuracy. Another model (named E'') increases accuracy by nearly 2\% with fewer parameters, but at a 3$\times$ increase in FLOPS.

We also selected a handful of promising models after 90 epochs to train for the full 350 epochs, in order to match the evaluation in the original EfficientNet paper. Under this setting, the base EfficientNet-B0 architecture reports a final accuracy of 75.2\%. Model E' has a final accuracy of 74.3\% (but maintains the 7--8\% improvement on FLOPS and parameter count). Model E'' outperforms the base architecture with fewer parameters and a higher final accuracy of 76.9\%.

\begin{figure}[t]
    \centering
    \includegraphics[width=.8\linewidth,trim={0 0 0 2cm},clip]{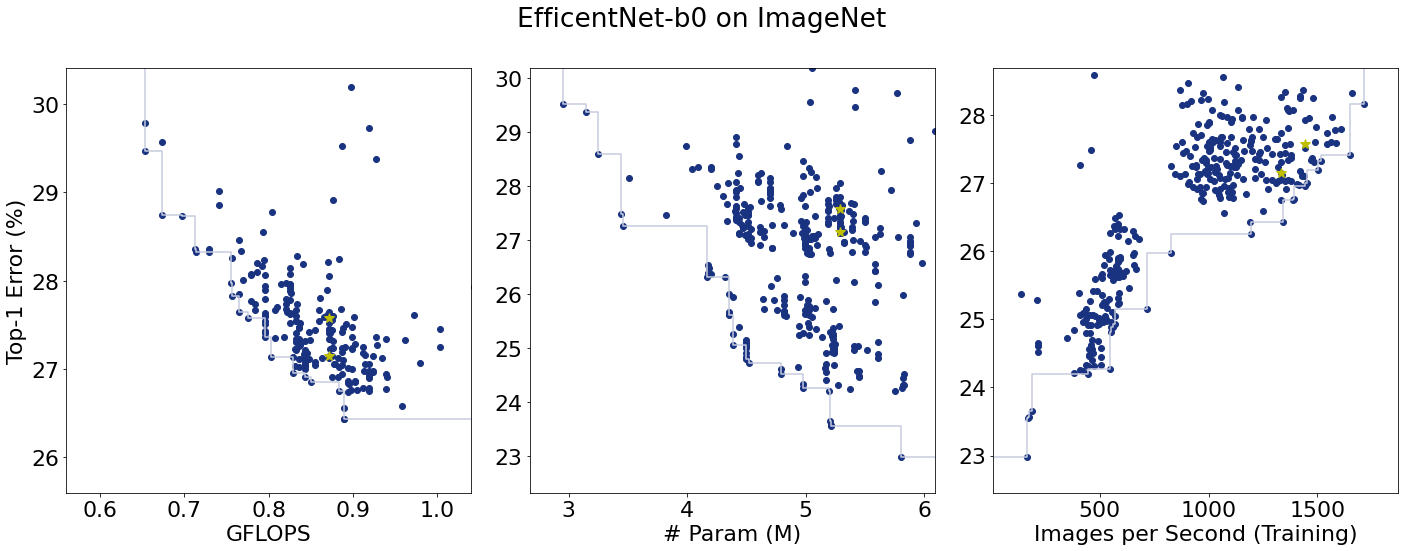}
    \caption{Evolving an EfficientNet-B0 model on ImageNet. We optimize for increased accuracy, decreased FLOPS, and increased training speed. We have decreased GFLOPS by 9\% and parameter count by 20\% for comparable accuracy and speed.} 
    \label{fig:efficientnet_imagenet}
\end{figure}

\subsection{Ablation Studies}
\label{sec:ablation}

This section evaluates the importance of key elements in the design of \ourname{}. For these experiments, we used ResNet-34 as our base model and CIFAR-10 for the dataset.

\subsubsection{Program Properties}
\label{sec:ablation:properties}
To evaluate the importance of program properties, we compared \ourname{} with a similar strategy that replaces a randomly selected subgraph with a random subgraph (ignoring program properties), called the \emph{random subgraph strategy}. On average, both the original and replacement subgraphs are of size 3. Unlike \ourname{}, this strategy searches directly in the concrete space of architectures, without performing property inference and guided synthesis. This study tests the hypothesis that using abstract properties to guide the search enables \ourname{} to perform larger mutations while still evolving high-quality candidates. As shown in \cref{fig:baseline_vs_random_resnet34}, \ourname{}'s accuracy-parameters Pareto curve completely dominates the random subgraph strategy's by a wide margin; specifically, the random subgraph strategy struggles to evolve models with high accuracy, and in fact \textit{makes next to no progress on improving the base architecture}. These results confirm our intuition that (1) small steps in the abstract search space can be effectively translated into large steps in the concrete space while maintaining good performance, whereas (2) large random steps directly in the concrete space are dominated by poor quality architectures. Additional results in \cref{appendix:ablation:random_subgraph} further demonstrate that the performance of the random search degrades rapidly as the size of the mutations increases.

We also evaluated the importance of each individual program property by modifying \ourname{} to use only a single property throughout the search. Our results show that all three program properties are important, as using them together outperforms using any individual property alone. More details can be found in \cref{appendix:ablation:properties}.

\subsubsection{Progressive Synthesis Algorithm}
\label{sec:ablation:synthesis}

Next, we evaluate the benefit of having an efficient synthesis algorithm. In this experiment, we compare with a na\"ive enumerative algorithm that simply enumerates all possible subgraphs of a given size in a random order, starting from a size of 1 and increasing the size until it finds a subgraph that satisfies the target program properties. As expected, the enumerative algorithm is highly inefficient, taking an average of 1,932 seconds to synthesize a satisfying subgraph, compared to only 67 seconds when using our progressive synthesis algorithm; for context, training a ResNet-34 model to evaluate its accuracy in our setting takes roughly 600 seconds, so enumerative synthesis introduces over 300\% overhead. Efficiency notwithstanding, the \emph{quality} of candidates explored by the two variants are similar when evaluating the same number of candidates (see \cref{fig:baseline_vs_enum_resnet34} in \cref{appendix:ablation:synthesis}), which suggests that the constraints imposed by the abstract properties are largely responsible for the quality of our search.

\begin{figure*}[t]
    \centering
    \includegraphics[width=.8\linewidth,trim={0 0 0 2cm},clip]{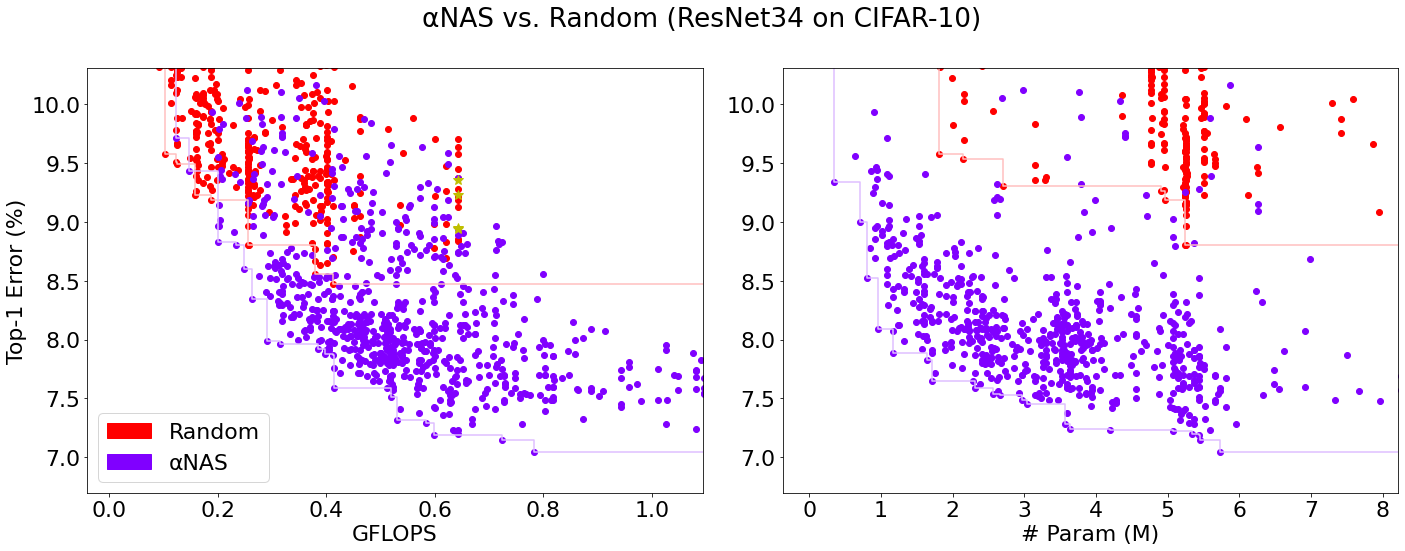}
    \caption{Evolving a ResNet-34 model on CIFAR-10 using \ourname{} vs. a random subgraph mechanism.} 
    \label{fig:baseline_vs_random_resnet34}
\end{figure*}

\subsection{Comparison with Existing NAS Techniques}
\label{sec:existing_approches}

This section compares \ourname{} against prior approaches that do not require defining a carefully-crafted structured search space. In particular, we adopt the strategies used in Primer \citep{so2021primer} and AutoML-Zero \cite{automl-zero}, two state-of-the-art NAS techniques for unstructured search spaces. These approaches search directly in the concrete space of architectures, similar to the random subgraph strategy in \cref{sec:ablation:properties}, but compensate by using much smaller mutations. Primer mutates on the basis of individual operations one at a time, by either deleting, inserting, or swapping operations; or randomly changing one parameter value. 
AutoML-Zero considers an additional mutation that randomizes an entire component, essentially combining mutations used in Primer and the random subgraph strategy. 

For clarity, we emphasize that our evaluation is only meant to imitate their search mechanisms (i.e., mutation strategies) within our evolutionary framework. In particular, both Primer and AutoML-Zero focus on slightly different tasks and hence use different program representations compared to ours. They also implement other advanced techniques, which are orthogonal to the contributions in this work, to reduce candidate evaluation time (e.g., early stopping). For instance, AutoML-Zero was originally designed to evolve both the neural network architecture and optimization algorithm jointly from scratch. Since this strategy does not yield results comparable to even our baseline (i.e., ResNet-34 on CIFAR-10), in order to establish a more fair comparison, we instead apply the respective search mechanisms to evolve architectures starting from the same base architecture.

According to the results in \cref{fig:baseline_vs_primer_resnet34,fig:baseline_automl_zero_resnet34} in \cref{appendix:automl_zero}, \ourname{} discovers significantly more Pareto optimal models than both Primer and AutoML-Zero.
Due to using only small mutations, Primer requires many more evaluations to make large changes to the base architecture, leading to slower exploration compared to AutoML-Zero and \ourname{}; the slow exploration is visible in \cref{fig:baseline_vs_primer_resnet34} as Primer does not discover architectures with GFLOPS lower than 0.3.
The distribution of the candidates explored by AutoML-Zero is between Primer's and the random subgraph strategy’s, which is to be expected as AutoML-Zero's mutations are the union of those two's.
In summary, these results further support the benefit of using properties to guide the search over an abstract design space, as \ourname{} delivers a consistent and significant advantage over existing strategies that search directly over the concrete space of actual architectures.

\begin{figure*}[t]
    \centering
    \includegraphics[width=.8\linewidth,trim={0 0 0 2cm},clip]{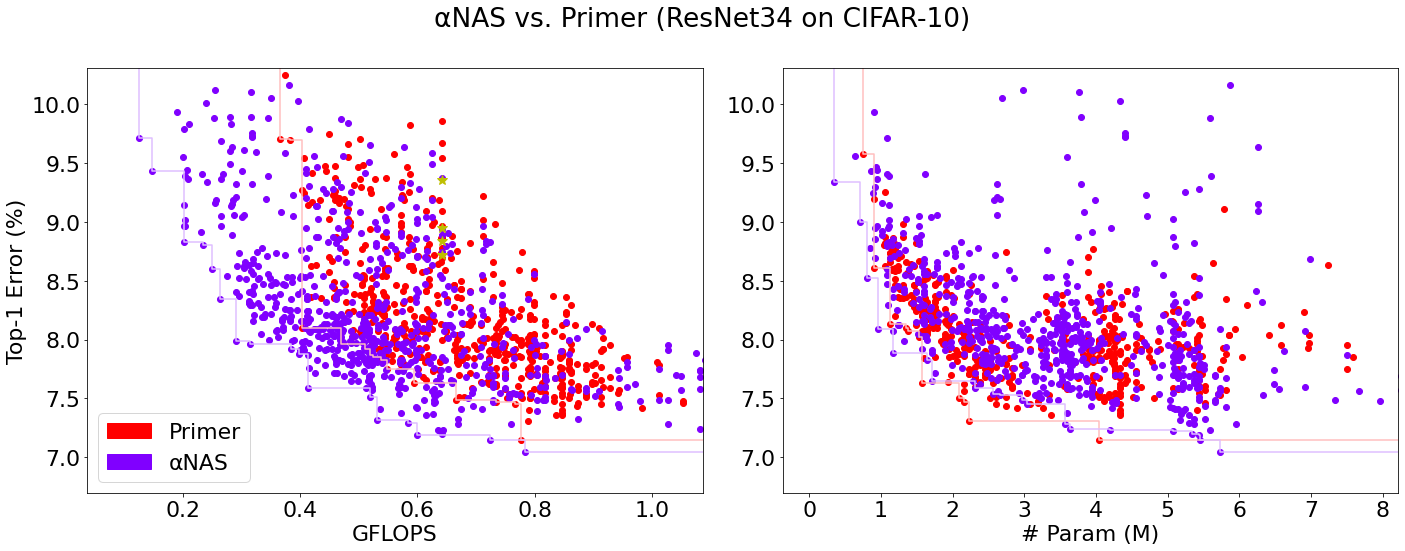}
    \caption{Evolving a ResNet-34 model on CIFAR-10 using \ourname{} vs. the Primer search mechanism.} 
    \label{fig:baseline_vs_primer_resnet34}
\end{figure*}

Additionally, it is worth comparing our result against \citet{turner-asplos}, which is another approach that applies techniques in program languages (i.e., compiler optimizations) to NAS. As we cannot reimplement \citet{turner-asplos}'s compiler optimizations easily in our framework, we compare our results with those reported in their paper directly. On the CIFAR-10 dataset, starting from ResNet-34, \citet{turner-asplos} is able to compress the model by 2--3$\times$ without accuracy loss after proposing 1000 candidates (they reject approximately 90\% of their candidates without training using a heuristic that estimates the likelihood a candidate architecture will have significantly degraded accuracy). In contrast, \ourname{} is able to compress the same model by over 26$\times$ (\cref{sec:cifar10}) within 800 candidates. Finally, we note that the architectures produced by \citet{turner-asplos} are less interpretable than ours as their optimization occur at the level of loop transformations, and hence the resulting architectures cannot be expressed in terms of the standard human-readable neural network primitives.

\section{Related Work}
\label{sec:related_work}

\subsection{Neural Architecture Search}

This section describes other existing NAS approaches related to our work that have not been discussed previously.
Somewhat similar to our property-guided NAS, some approaches restrict their search spaces to only mathematically structure-preserving transformations, e.g., only mutating the operation's parameters or changing network connections \cite{network-morphism,efficient-transformation}. Their main motivation is to be able to transfer already-trained weights from an existing architecture to new architectures, so as to accelerate the candidate evaluation process; in contrast, our properties are motivated by the semantics of deep neural networks. Their properties are also more strict than ours, thus limiting their search spaces.

Our work is one of the first to apply techniques programming languages to NAS. A prior work by \citet{turner-asplos} combines non-semantics preserving transformations from NAS with semantics preserving transformations from compiler optimizations. In contrast, our work exploits other higher-level program languages techniques including program synthesis, program abstraction, and abstract interpretation.

\subsection{Optimizing Neural Networks for Efficiency}

Several techniques such as quantization \cite{quantize-survey,quantize-banner,quantize-wang} and pruning \cite{pruning-hassibi,pruning-lecun,pruning-luo,pruning-frankle} aim to improve execution time efficiency. Note that these approaches are orthogonal to our work in the sense that one can use our technique to generate new architectures, and then prune or quantize the weights to compress the model once trained. Similarly, the primary benefit of auto-parallelization \cite{flexflow,Gpipe2018cvpr,PipeDream2019sosp}, graph substitutions \cite{metaflow,taso,tensat}, and other tensor compiler optimizations \cite{tvm_osdi_18,autohalide,tensor-comprehensions,taco,go2020neurips,xtat} is that they generate semantic-preserving optimizations and can be (theoretically) applied to any neural network architecture. In contrast, our method aims to discover semantically novel architectures.

\subsection{Goal-Directed Program Synthesis}

Our key idea of using program properties to guide the synthesis of neural network subgraphs can be classified as goal-directed program synthesis. A number of prior works use abstract properties of partial programs to direct their synthesis procedure \cite{flashmeta,lens-asplos,blaze-popl,scythe-pldi}. The main difference is that we develop a precise notion of distance between partial programs and properties; thus, we can devise a linear-time algorithm to synthesize satisfying programs. Our setting also departs from most program synthesis tasks in that the design of neural networks does not have a strict correctness requirement.

\section{Conclusion and Future Work}

With the increasing complexity of hardware accelerators with diverse performance characteristics, a large variety of ML compilers that target them, and a wide range of well-known neural network architectures, finding the right architecture to solve a task is challenging and often requires a deep knowledge of the full stack. Techniques like neural architecture search (NAS) that automate the exploration of this complex design space are therefore likely to be increasingly more important. In this work, we identified some of the fundamental limitations of current approaches for NAS, and proposed a principled approach for overcoming these limitations using a novel property guided synthesis strategy. We demonstrate that a simple evolutionary search over the resulting abstract search space can produce significantly improved architectures on several widely used benchmarks for image recognition, compared to the base architecture used to initialize the search. One possible direction for future work would be to to initialize the search with multiple base architectures, and let the evolutionary mechanism decide which individuals to optimize---possibly with more sophisticated evolutionary techniques such as crossover. Additionally, as our approach is broadly applicable to most (if not all) deep neural net architectures and domains, we also leave to future work the evaluation of other interesting tasks such as language modeling or audio speech recognition.

\begin{acks}
We would like to thank Chen Liang, David So, Esteban Real, Hanxiao Liu, Hyeontaek Lim, Omar Costilla-Reyes, Martin Abadi, Mike Burrows, Ras Bodik, Rishabh Singh, and Yanqi Zhou for their help and insightful feedback during the development of this project.
\end{acks}

\bibliography{ref}

\appendix
\newpage

\section{Implementation Details}

\subsection{Primitive Operations}
\label{appendix:primitives}

\cref{tab:primitives} presents the set of primitive operations we currently support in our framework. Note that our framework supports evolving models with operations outside the table, but our subgraph selection mechanism will only select subgraphs containing operations in the table (and our synthesis procedure will only synthesize subgraphs using the operations in the table).

\begin{table}[h]
    \centering
    \footnotesize
    \begin{tabular}{lp{4cm}p{5cm}}
    \toprule
    Name & Description & Parameters \\
    \midrule
    Dense & Dense layer applied to the output features  & output features \\ 
    Convolution & Convolution operator with same padding & output features, strides (either 1 or equal to kernel), kernel shape (always square) \\
    Grouped Convolution & Grouped convolution operator with same padding & output features, strides (either 1 or equal to kernel), kernel shape (always square), feature groups > 1\\
    Dilated Convolution & Dilated convolution operator with same padding & output features, strides (either 1 or equal to kernel), kernel shape (always square), input dilation > 1 \\
    Add & Adds two input tensors & \\
    Scalar Multiply & Multiply input by a scalar & scalar value \\
    ReLU & ReLU activation & \\
    GeLU & GeLU activation & \\
    SiLU & SiLU activation & \\
    Sigmoid & Sigmoid activation & \\
    Softmax & Softmax layer & \\
    Batch Norm & Batch normalization & \\
    Layer Norm & Layer normalization & \\
    Group Norm & Group normalization & number of groups \\
    Dropout & Dropout layer & dropout rate \\
    Average Pool & Average pooling layer & window dimension (always square) \\
    Max Pool & Max pooling layer & window dimension (always square) \\
    \bottomrule
    \end{tabular}
    \caption{Primitive operations that we use to synthesize subgraphs}
    \label{tab:primitives}
\end{table}

\subsection{Structure Mutations for Subgraphs with Multiple Inputs and Outputs}
\label{appendix:structure_mutation}

This section provides a high level description of how to support structure mutations for subgraphs with multiple inputs and outputs. Recall that our synthesis algorithm first decomposes a subgraph with reshapes and binary operations into sequential subgraphs connected by reshape and binary operations. When then mutate each sequential subgraph individually.

\subsubsection{Reshape Operators}

The main challenge with synthesizing sequential subgraphs containing reshape operators is that in satisfying the mixing property, we can no longer rely on Corollary \ref{cor:matrix_semiring}, which is defined only for square matrices. For instance, if the input has 3 dimensions, but the output has 4 dimensions, it is not clear a priori how to associate the input dimensions with the output dimensions without fixing in advance the sequence of reshape operations.

One approach is to always synthesize reshape operators in pairs. Every reshape operator $t$ (e.g., transpose, flatten) has a corresponding inverse $t^{-1}$ such that $t \circ t^{-1} = \text{id}$. Hence, we can always ``undo'' a reshape operator if it makes the synthesis problem infeasible. 
After applying the reshape operator $t$, we can simultaneously transform the target property via $t^{-1}(u)$ and continue with synthesis---this new problem is still guaranteed to be feasible (and will have the same distance as before). At a later point, we then need to reverse the transformation by applying $t^{-1}$ to the program and reverting back to $u$ as the target property. This guarantees that the original properties are satisfied.

In some cases, we may have that $d(t(p), u) < \infty$, i.e., the synthesis problem is feasible without reversing the reshape operation. In this scenario, it is safe to continue synthesis without reversing the transformation; however, we should also limit the number of such reshape operators as there is no guarantee that $d(t(p), u) \le d(p, u)$ (which violates the assumption of progress).

\subsubsection{Support for Binary Operations}

We next describe how to support the insertion and deletion of binary operations. Namely, whenever we choose to synthesize a binary operation as the next transformation, we must also synthesize something to fill the other input slot before proceeding. To accomplish this, we first select a random starting point, which is either a previously generated operation, or an input to the subgraph, with the requirement that the shape property from the selected starting point to the binary input must be feasible. From here, a recursive call to the sequential synthesizer can generate a subgraph that connects the selected starting point to the missing input of the binary operation.

To delete a binary operator, we first check whether the operator is necessary by replacing one of its inputs with a placeholder (constant) input, and inferring the resulting properties with respect to its output. If the target properties are satisfied, then we can safely replace the binary operation with an identity operation, deleting the unneeded input. Note that one way a binary operation can become unnecessary is if a previous synthesis step already synthesized another binary operation elsewhere that satisfied the associated properties.

Finally, combining the ability to insert and delete binary operations yields the desired mutation of subgraph structures.

\subsection{Evolutionary Neural Architecture Search}
\label{appendix:evol}

In this section, we describe our evolutionary framework for NAS. Our design is based on the evolutionary algorithm of \citet{real2019regularized}, except that we do not perform the age regularization due to only running the evolutionary search for 800 individuals.

In order to support multi-objective optimization, we develop a notion of weighting individuals based on their Pareto optimality. A point is \textbf{Pareto optimal} with respect to multiple objectives if no other point is a strict improvement on all objectives.

We assume that all objectives trade-off against a single primary objective; in our case, the goal is always to maximize accuracy versus the various secondary objectives. We define the \textbf{Pareto weight} of an individual as the shortest distance (in the $\ell_2$ sense) from the individual to the Pareto curve, which is built by linearly interpolating the Pareto-optimal points. Lower Pareto weights correspond to better Pareto optimality; Pareto optimal points having a Pareto weight of 0. 

However, as this distance is not invariant to units, we also perform a normalization in each dimension. To compute the normalization factor, we simply compute the average slope of the full Pareto curve, which is equivalent to the slope of the two endpoints. The idea is that this slope captures an average notion of how difficult it is to trade off between the two objectives.

\begin{algorithm}[h]
\small
\caption{Evolutionary selection mechanism}
\label{alg:selection}
\begin{algorithmic}[1]
\Require set of candidates $M$, primary objective $o_p$, set of secondary objectives $O_s$, percentage $k$
\Ensure candidate to be mutated
  \State Randomly select a secondary objective $o_s \in O_s$
  \State Build the Pareto curve $C$ of the primary objective $o_p$ versus the secondary objective $o_s$
  \State Compute $W = \{w_i \ |\  w_i = $ Pareto weight of $m_i \in M$ with respect to $C\}$
  \State $M_{\text{top}} =$ top $k\%$ of $M$ according to $W$
  \State \Return Random $m \in M_{\text{top}}$
\end{algorithmic}
\end{algorithm}

\Cref{alg:selection} provides a summary of our selection mechanism.
This process is repeated independently each time a new individual is requested for mutation, with the current population of individuals as the input.

\section{Experimental Setup and Additional Experimental Results}
\label{appendix:more_results}

\subsection{Experimental Setup}
\label{appendix:setup}

\cref{tab:experiment_setup} summarizes the setup of our experiments.

For the ResNet architectures (ResNet-34, ResNet-50), we train for 90 epochs using SGD with base learning rate 0.1, momentum 0.9, and weight decay 0.0004. The learning rate warms up over 5 epochs then follows a cosine decay schedule. We also add a nonstandard dense layer between the head and the first block, to project the feature dimension from 3 to 64, so that all the blocks have the same shape signature and hence type (the first block is otherwise a singleton type).

For the ViT-S/16 architecture, we train for 90 epochs using the Adam optimizer, with base learning rate .003, beta1 = 0.9, and beta2 = 0.999. The learning rate warms up over 10,000 batches, then follows a cosine decay schedule down to .00001.

For the EfficientNet-B0 architecture, we follow the same setting as the original paper \cite{efficientnet}, except that we train for 90 epochs. We use the RMSProp optimizer, with a base learning rate of 0.016. The learning rate warms up over 5 epochs, then decays exponentially down to 0.012. We also use exponential moving average with decay 0.9999.

We use standard Inception-style data augmentation for training on ImageNet, and train and test at \mbox{224$\times$224} resolution.

\begin{table}[h]
{\footnotesize
\begin{tabular}{lllcccl}
\toprule
{\bf Experiment} & {\bf Dataset} & {\bf Base Model} & {\bf Accuracy} & {\bf Trials} & {\bf Epochs} & {\bf Secondary Objectives} \\
\midrule
Case Studies & CIFAR-10 & 2-layer CNN & 81.3\% & 400 & 90 & flops, parameter count \\
(\cref{sec:case_studies}) & & ResNet-34 & 91.5\% & 400 & 90 & flops, parameter count \\
\cmidrule{2-7}
             & ImageNet & ResNet-50 &    75.7\%    &  400   & 90 & flops, training throughput\\
             &          & ViT-S/16  &   70.4\%     &  400   & 90 & flops, training throughput \\
             &          & EfficientNet-B0 & 72.8\% &  400   & 90 & flops, training throughput \\
\midrule

Property Ablation & CIFAR-10 & ResNet-34 & 91.5\% & 800 & 90 & flops, parameter count \\
(\cref{appendix:ablation:properties},\\
\cref{sec:ablation:properties}, \\
\cref{appendix:ablation:random_subgraph}) \\
\midrule

Synthesis Ablation & CIFAR-10 & ResNet-34 & 91.5\% & 800$^\ddagger$ & 90 & flops, parameter count \\
(\cref{sec:ablation:synthesis}, \\
\cref{appendix:ablation:synthesis}) \\
\midrule

Comparison with & CIFAR-10 & ResNet-34 & 91.5\% & 800 & 90 & flops, parameter count \\
Existing Approaches \\
(\cref{sec:existing_approches}, \\
\cref{appendix:automl_zero}) \\
\bottomrule
\end{tabular}
}
\caption{Training setup for the experiments. We initialize the search with the base model and run the search for a given number of trials; each trial involves training a proposed candidate for a given number of epochs. the accuracy column indicates the accuracy of the base model on the dataset. $^\ddagger$The enumerative search comparison uses only 600 trials, due to the enumerative synthesis being very slow.}
\label{tab:experiment_setup}
\end{table}

\subsection{Random Subgraph Ablations}
\label{appendix:ablation:random_subgraph}

We additionally varied the expected size of the subgraph synthesized by the random subgraph strategy in \cref{fig:baseline_vs_random_resnet34_p20,fig:baseline_vs_random_resnet34_p50,fig:baseline_vs_random_resnet34_p75}. In particular, we select a random subgraph of expected size 3 for mutation. Then, we synthesize a single (random) operation. At this point, we stop and return the current subgraph with probability $1-p$. Otherwise, we synthesize another (random) operation and repeat. Hence, the expected subgraph size is $1 + 1/(1-p)$. For instance, the plot for $p=0.2$ synthesizes subgraphs of size roughly 2 on average, which represents fairly small mutations; correspondingly, we see that the performance is the best of the three random subgraph strategies, and also similar to the performance of the AutoML-Zero and Primer strategies. At $p=0.5$, the expected subgraph size is 3 and there is a rapid decline in quality. The worst setting is at $p=0.75$, with an expected random subgraph of size 5. Note that \cref{fig:baseline_vs_random_resnet34_p50} is the same as \cref{fig:baseline_vs_random_resnet34} in the main text.

\cref{fig:baseline_vs_random_resnet34_p20,fig:baseline_vs_random_resnet34_p50,fig:baseline_vs_random_resnet34_p75} demonstrate that the performance of random search degrades very rapidly as we increase the (expected) size of the mutation from 2 to 5, to the point where there is practically no progress on improving either FLOPS or parameter count at $p=0.75$. In contrast, \ourname{} with property-guided synthesis consistently discovers mutations that improve performance at size $p=0.75$ in the same setting (and outperforms the random search at all mutation sizes).

\begin{figure*}[htpb]
    \begin{subfigure}[b]{\textwidth}
        \centering
        \includegraphics[width=\linewidth,trim={0 0 0 2cm},clip]{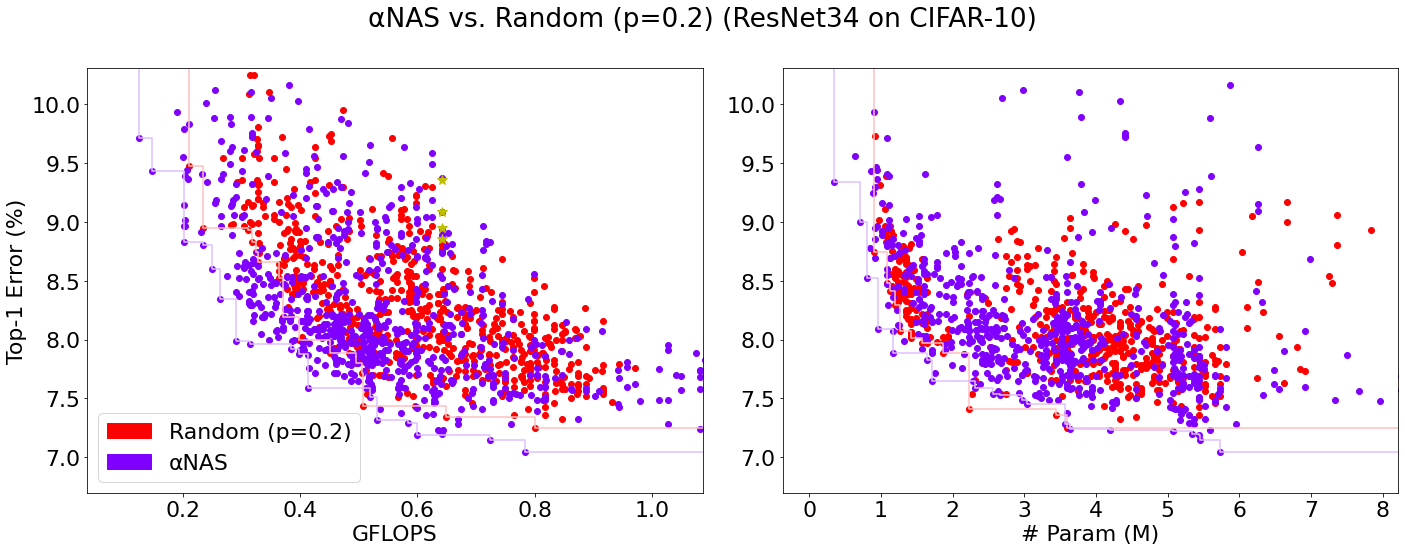}
        \caption{Evolving a ResNet-34 model on CIFAR-10 using \ourname{} vs. a random subgraph mechanism with p=.20.}
        \vspace{2em}
        \label{fig:baseline_vs_random_resnet34_p20}
    \end{subfigure}
    
    \begin{subfigure}[b]{\textwidth}
        \centering
        \includegraphics[width=\linewidth,trim={0 0 0 2cm},clip]{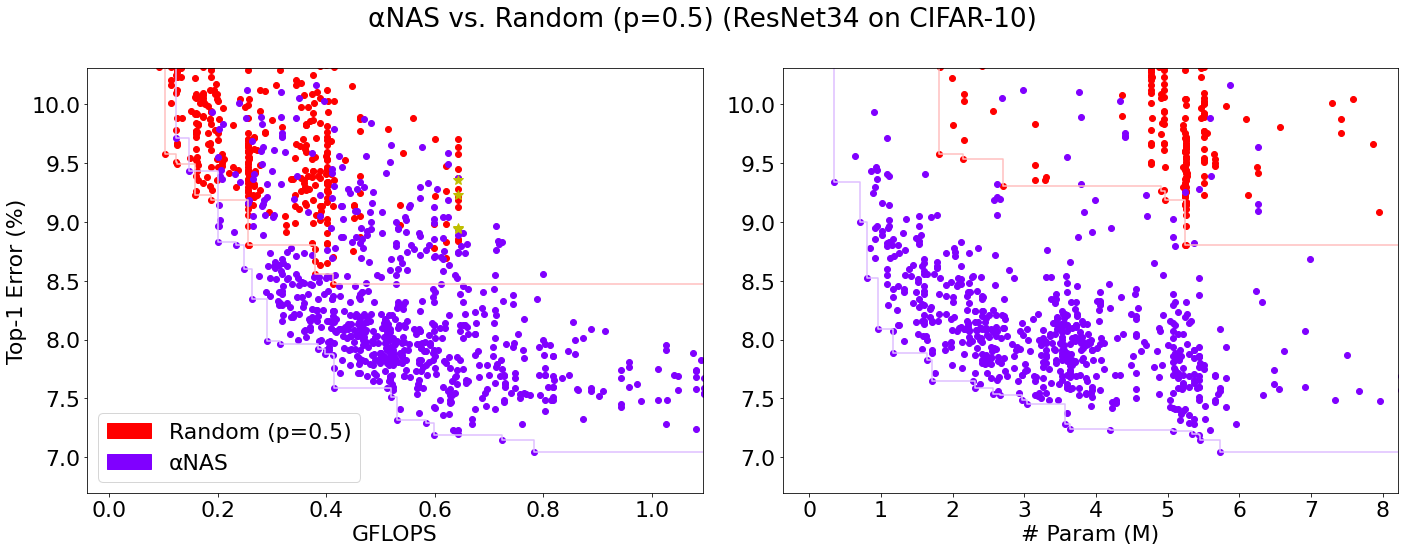}
        \caption{Evolving a ResNet-34 model on CIFAR-10 using \ourname{} vs. a random subgraph mechanism with p=.50.}
        \vspace{2em}
        \label{fig:baseline_vs_random_resnet34_p50}
    \end{subfigure}
    
    \begin{subfigure}[b]{\textwidth}
        \centering
        \includegraphics[width=\linewidth,trim={0 0 0 2cm},clip]{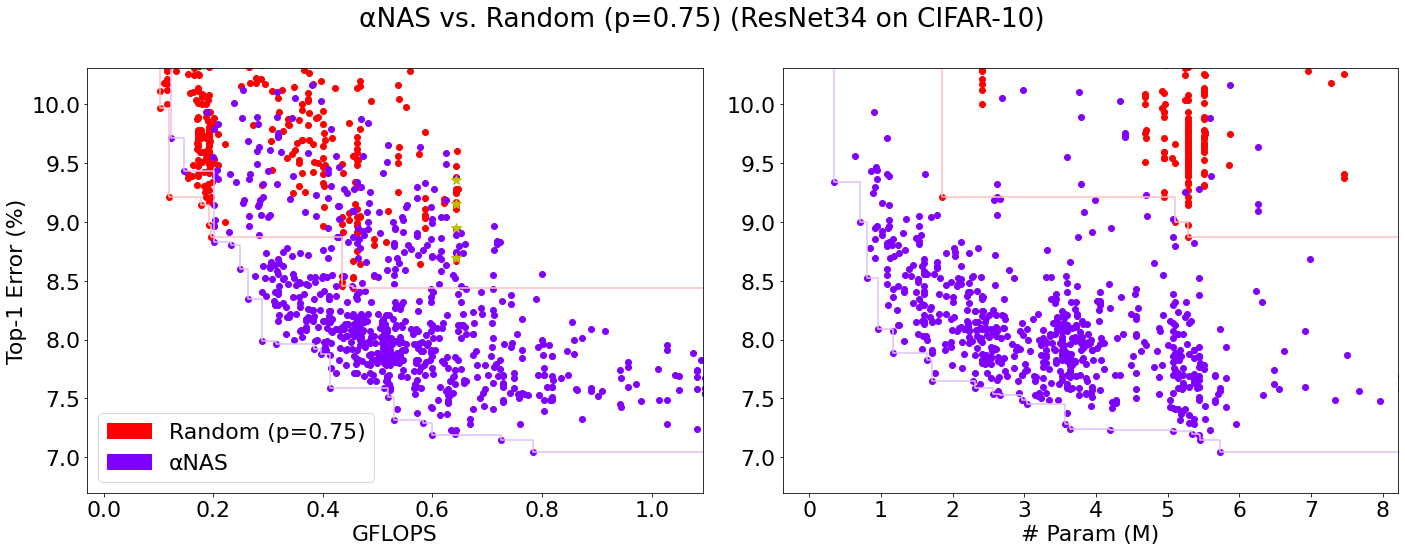}
        \caption{Evolving a ResNet-34 model on CIFAR-10 using \ourname{} vs. a random subgraph mechanism with p=.75.} 
        \label{fig:baseline_vs_random_resnet34_p75}
    \end{subfigure}
\caption{Evolving a ResNet-34 model on CIFAR-10 using \ourname{} vs. a random subgraph mechanism with varying expected sizes}
\end{figure*}

\FloatBarrier
\clearpage

\subsection{Program Property Ablations}
\label{appendix:ablation:properties}

To evaluate the contributions of the proposed program properties, we compare \ourname{} using the standard properties and mutations with the modified variants in \cref{fig:ablation_variants}. For this experiment, we train a ResNet-34 on CIFAR-10. For all the variants, the synthesized subgraph must be at least the size of the original subgraph minus 2. The results of our comparisons are shown in \cref{fig:baseline_vs_no_prop_resnet34,fig:baseline_vs_mixing_only_resnet34,fig:baseline_vs_depth_only_resnet34,fig:baseline_vs_shape_only_resnet34,fig:baseline_vs_no_mutate_resnet34}. Using a single program property allows the search to discover architectures closer to Pareto optimal compared to using no properties, but using a single property is still visibly inferior to the baseline. When using all three properties without mutations, the outcome is quite similar to the baseline in the medium to high FLOPS region. We attribute this behavior to the search being biased toward proposing more expensive architectures without property mutations: because a subgraph satisfies the target properties so long as its properties exceed the target, without mutations to decrease property values, the target property values can only increase over time.

\begin{table}[h]
{\footnotesize
\begin{tabular}{lp{8cm}}
\toprule
{\bf Variant} & {\bf Description} \\
\midrule
No Properties & Use no program properties to guide synthesis (same as the random subgraph strategy).\\
Mixing Only & Use only the mixing property to guide synthesis. \\
Depth Only & Use only the depth property to guide synthesis. \\
Shape Only & Use only the shape property to guide synthesis. \\
No Mutate & Use all three properties to guide synthesis but do not mutate properties. \\
Baseline & Use all three properties to guide synthesis and mutate properties. \\
\bottomrule
\end{tabular}
}
\caption{
Different variants of \ourname{} to evaluate the benefit of each program property.}
\label{fig:ablation_variants}
\end{table}

\begin{figure*}[!h]
    \centering
    \includegraphics[width=\linewidth,trim={0 0 0 2cm},clip]{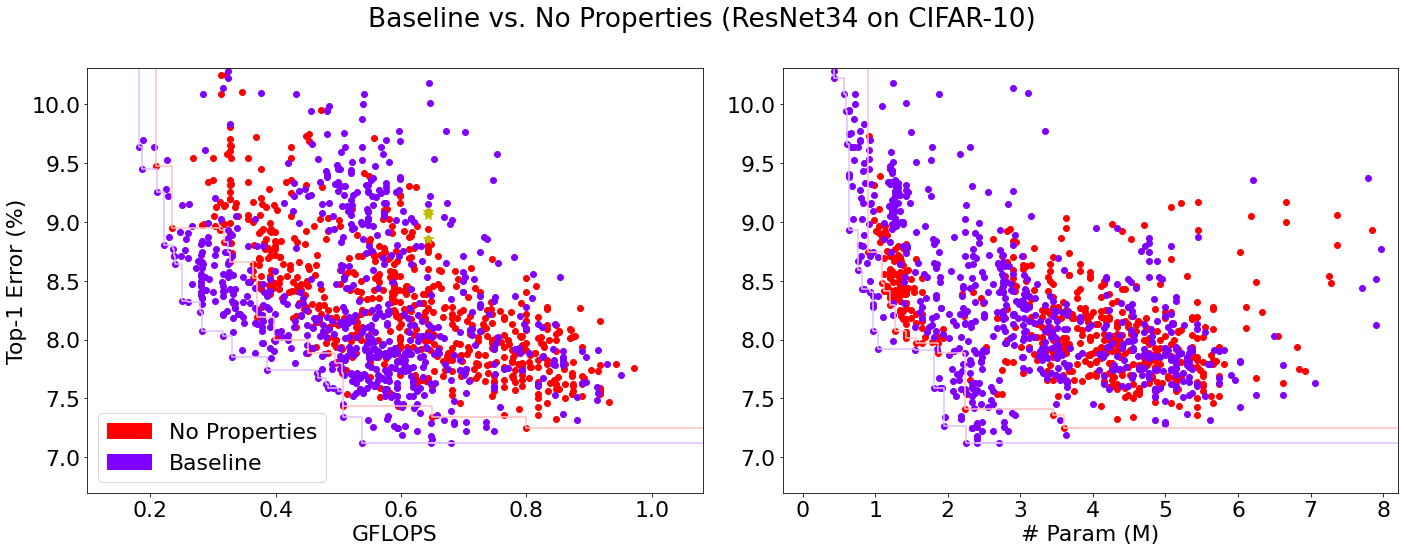}
    \caption{Evolving a ResNet-34 model on CIFAR-10 using the baseline approach vs. no properties.} 
    \vspace{2em}
    \label{fig:baseline_vs_no_prop_resnet34}
\end{figure*}

\begin{figure*}[!h]
    \centering
    \includegraphics[width=\linewidth,trim={0 0 0 2cm},clip]{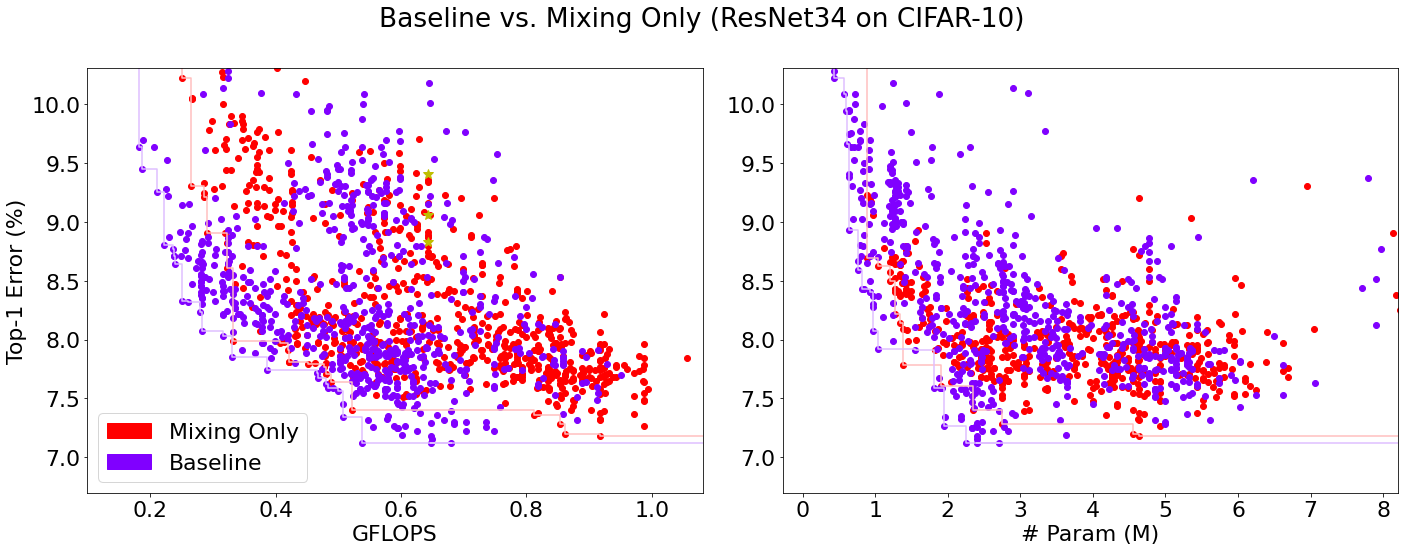}
    \caption{Evolving a ResNet-34 model on CIFAR-10 using the baseline approach vs. mixing property only.} 
    \vspace{2em}
    \label{fig:baseline_vs_mixing_only_resnet34}
\end{figure*}

\begin{figure*}[!h]
    \centering
    \includegraphics[width=\linewidth,trim={0 0 0 2cm},clip]{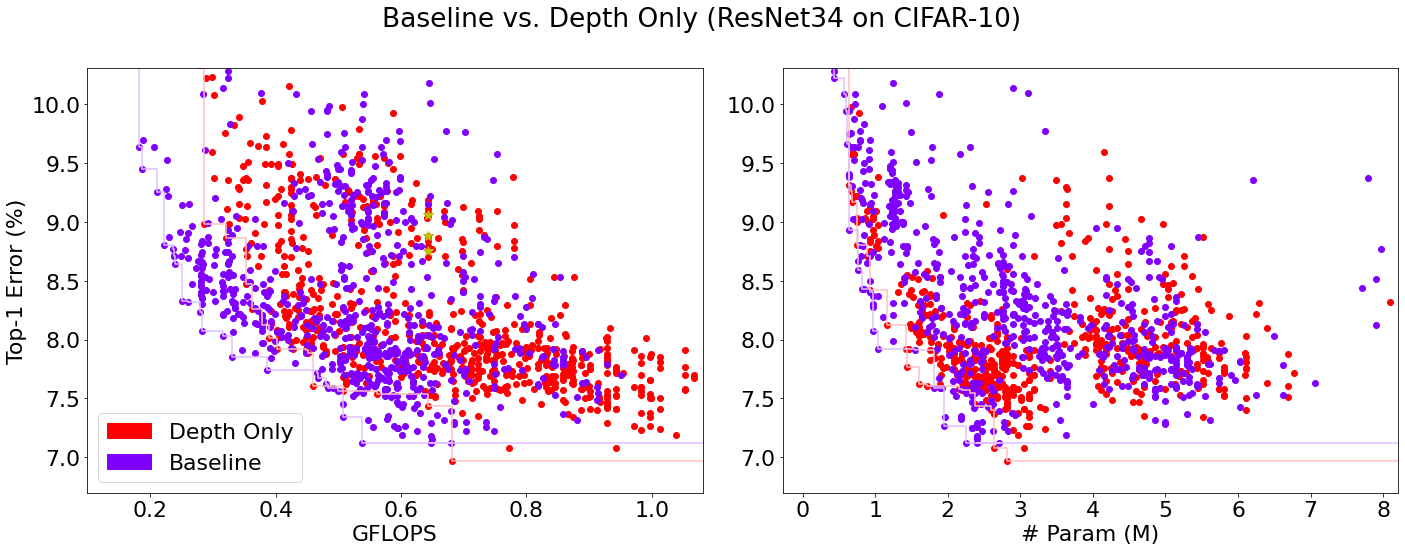}
    \caption{Evolving a ResNet-34 model on CIFAR-10 using the baseline approach vs. depth property only.} 
    \vspace{2em}
    \label{fig:baseline_vs_depth_only_resnet34}
\end{figure*}

\begin{figure*}[!h]
    \centering
    \includegraphics[width=\linewidth,trim={0 0 0 2cm},clip]{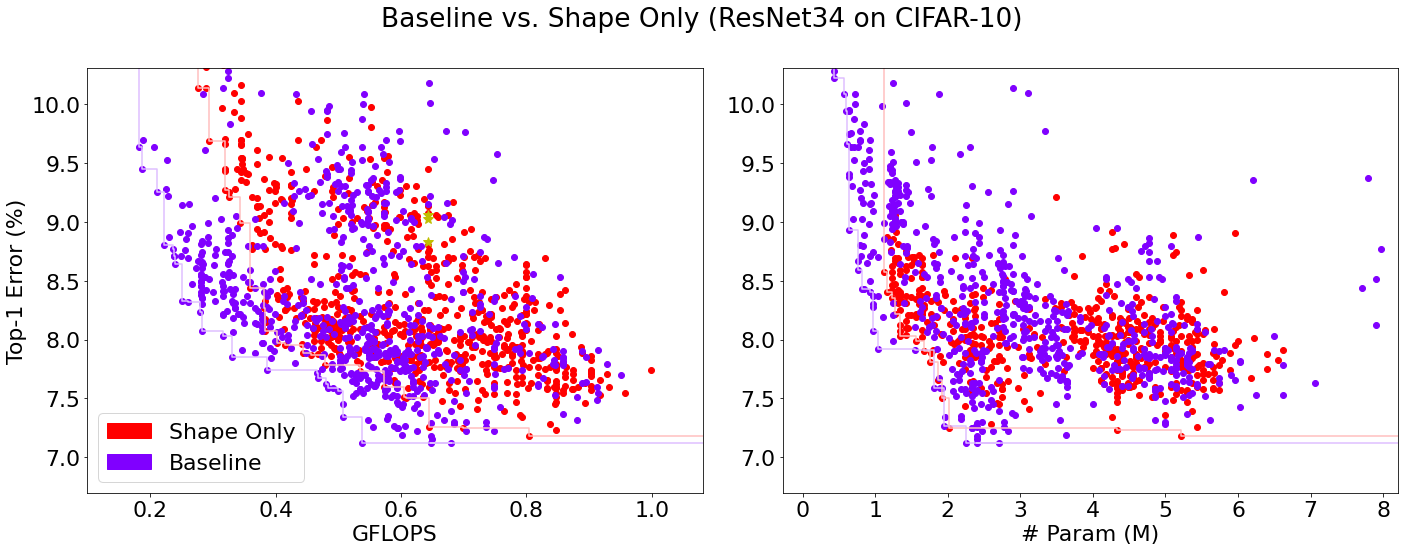}
    \caption{Evolving a ResNet-34 model on CIFAR-10 using the baseline approach vs. shape property only.} 
    \vspace{2em}
    \label{fig:baseline_vs_shape_only_resnet34}
\end{figure*}

\begin{figure*}[!h]
    \centering
    \includegraphics[width=\linewidth,trim={0 0 0 2cm},clip]{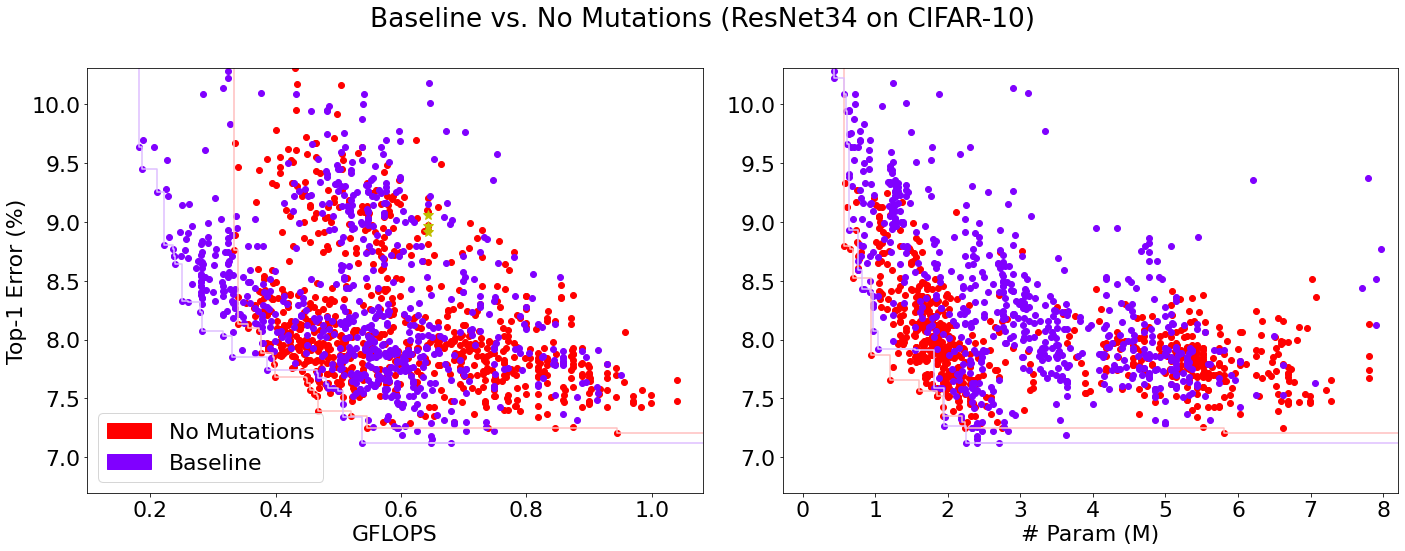}
    \caption{Evolving a ResNet-34 model on CIFAR-10 using the baseline approach vs. no mutations.} 
    \vspace{2em}
    \label{fig:baseline_vs_no_mutate_resnet34}
\end{figure*}

\FloatBarrier

\subsection{Progressive vs. Enumerative Synthesis}
\label{appendix:ablation:synthesis}

In this section, we compare the performance of \ourname{} when using progressive vs. enumerative synthesis to perform the concretization step as described in \cref{sec:ablation:synthesis}. In theory, both synthesis algorithms should produce similar outcomes in terms of the quality of the candidates explored, as they use the same set of abstract properties to guide the search. The main difference is that the enumerative synthesis is much slower than the progressive synthesis.

\cref{fig:baseline_vs_enum_resnet34} presents the distributions of the candidates explored using the two synthesis algorithms. Since enumerative synthesis always returns the smallest subgraph that satisfies the target properties, while progressive synthesis may not, enumerative synthesis is biased toward smaller subgraphs. Nevertheless, the performance of enumerative synthesis is substantially similar to the performance of progressive synthesis (and is by far the most similar of the other synthesis approaches we consider).

\FloatBarrier
\begin{figure*}[h]
    \centering
    \includegraphics[width=\linewidth,trim={0 0 0 2cm},clip]{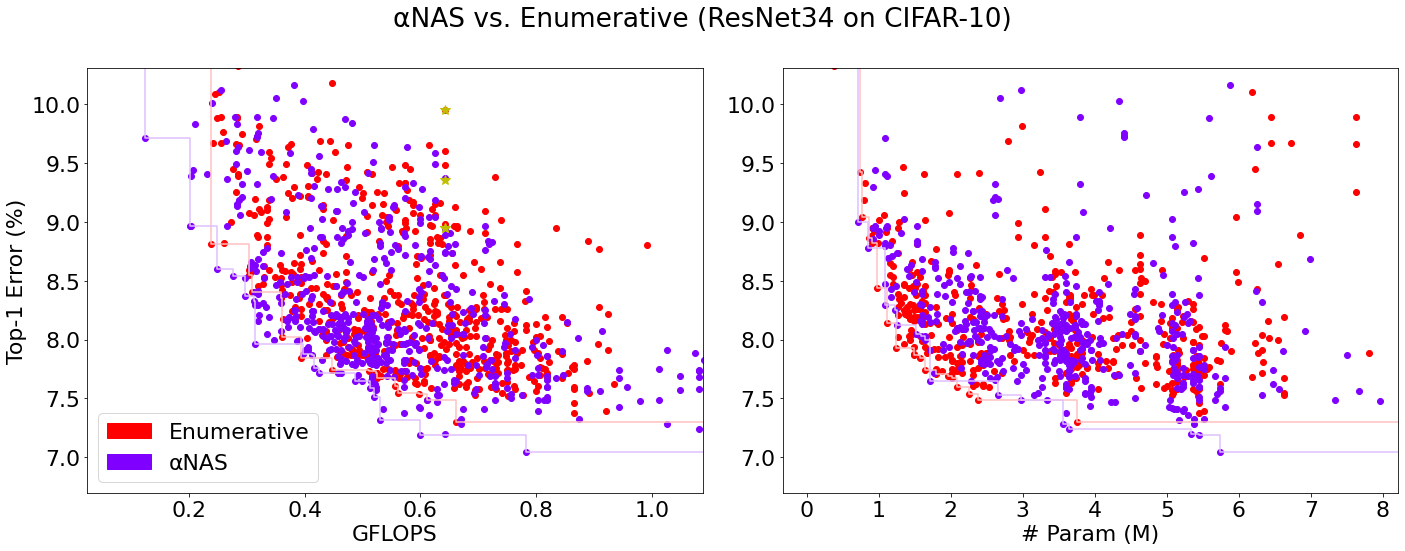}
    \caption{Evolving a ResNet-34 model on CIFAR-10 using \ourname{} vs. the enumerative search mechanism. Due to enumerative synthesis being slow, we compare it with \ourname{} over 600 trials.} 
    \label{fig:baseline_vs_enum_resnet34}
\end{figure*}

\subsection{Comparison with AutoML-Zero}
\label{appendix:automl_zero}

\cref{fig:baseline_automl_zero_resnet34} shows the result when comparing \ourname{} and AutoML-Zero as described in \cref{sec:existing_approches}.

\begin{figure*}[h]
    \centering
    \includegraphics[width=\linewidth,trim={0 0 0 2cm},clip]{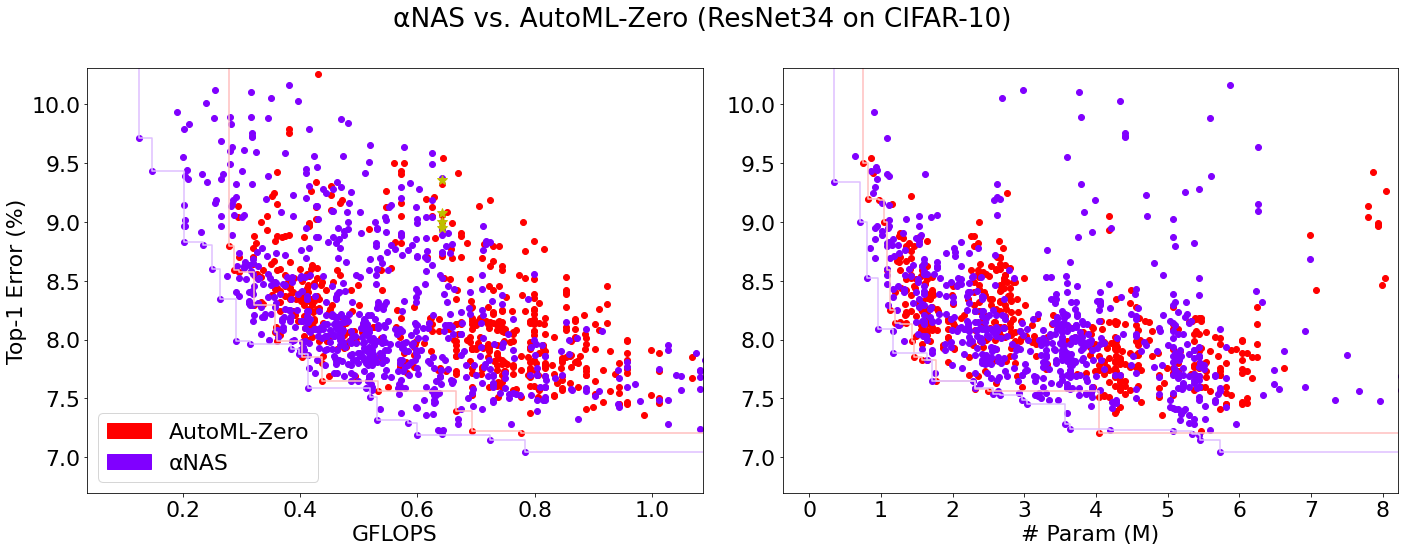}
    \caption{Evolving a ResNet-34 model on CIFAR-10 using \ourname{} vs. the AutoML-Zero search mechanism.} 
    \label{fig:baseline_automl_zero_resnet34}
\end{figure*}

\FloatBarrier
\clearpage

\newpage
\section{Abstract Progressive Synthesis}
\label{appendix:abstract_prog_syn}

We first redefine the relevant notions using the lens of abstract interpretation:

\begin{definition}
Given a program property $\Pi = (\mathcal{V}, \le, \alpha)$ and an abstract interpretation $\alpha$, an \textbf{abstract distance function} is a function $d^\alpha : \mathcal{V} \times \mathcal{V} \mapsto \mathbb{R}^+ \cup \{\infty\}$ such that:
\begin{enumerate}
    \item $d^\alpha(u, v) \ge 0$ for all $u, v \in \mathcal{V}$. 
    \item $d^\alpha(u, v) = 0 \iff u \ge v$.
    \item $d^\alpha(u, v) = \infty$ $\iff$ there exists $p \in \mathcal{P}$ such that (1) $p \models u$, but (2) there does not exist any finite sequence of transformations $t = t_1 \circ t_2 \circ \cdots \circ t_n$ such that $t(p) \models v$.
\end{enumerate}
\end{definition}

\begin{definition}
Given a space of programs $\mathcal{P}$, a program property $\Pi = (\mathcal{V}, \le, \alpha)$, an abstract interpretation $\alpha$, and an abstract distance function $d^\alpha$, a set of transformations $T \subseteq \mathcal{T}$ is an \textbf{abstract covering} if for all $u, v \in \mathcal{V}$ with $d^\alpha(u, v) > 0$, either (1) there exists $t \in T$ such that $d^\alpha(t^\alpha(u), v) < d^\alpha(u, v)$ or (2) $d^\alpha(u, v) = \infty$.
\end{definition}

\begin{definition}
An abstract covering $T$ is a \textbf{uniform abstract covering} if there exists a constant $\epsilon > 0$ such that for all properties $u, v \in \mathcal{V}$ with $d^\alpha(u, v) > 0$, there exists $t \in T$ such that
\begin{align}
    d^\alpha(t^\alpha(u), v) + \epsilon \le d^\alpha(u, v)
\end{align}
\end{definition}

We have the following guarantees for the abstract version of the greedy progressive synthesis algorithm, where we instead provide as input an initial program property value $\alpha(p_0) = u$, a target property value $v$, an abstract distance function $d^\alpha$, and an abstract uniform covering $T$:

\begin{theorem}
\label{thm:prog_syn_abstract}
Given an abstract distance function $d^\alpha$ and an abstract uniform covering $T$, the progressive synthesis algorithm satisfies the following three properties:
\begin{enumerate}
    \item Soundness: if the algorithm returns $u^* = (t^\alpha_n \circ \cdots \circ t^\alpha_1)(u)$, then $u^* \ge v$.
    \item Completeness: if the algorithm returns $\not\models$, then the synthesis task is infeasible.
    \item Progress: the algorithm terminates after a finite amount of time, which is linear in the length of the final sequence of transformations and linear in $|T|$.
\end{enumerate}
\end{theorem}

We omit the proof as it is identical to the proof of \cref{thm:prog_syn} 
except with the abstract semantics $t^\alpha$ substituted for the concrete semantics $t$.
Note that the only difference with the statement of \cref{thm:prog_syn} 
is the first soundness property, which is stated in terms of the abstract semantics $u^* \ge v$ instead of the corresponding concrete semantics $p^* = (t_n \circ \cdots \circ t_1)(p_0)$. The key is that the soundness property of the abstract interpretation allows us to translate this into the required form:

\begin{lemma}
If $\alpha(p_0) = u$ and $d^\alpha((t_n^\alpha \circ \cdots \circ t_1^\alpha)(u), v) = 0$, then $p^* = (t_n \circ \cdots \circ t_1)(p_0) \models v$.
\end{lemma}

\begin{proof}
By definition, $d^\alpha((t_n^\alpha \circ \cdots \circ t_1^\alpha)(u), v) = 0$ implies that $(t_n^\alpha \circ \cdots \circ t_1^\alpha)(u) = u^* \ge v$. By the soundness property of abstract interpretations, since $p_0 \models u$, it follows that $p^* = (t_n \circ \cdots \circ t_1)(p_0) \models u^*$. Putting these two facts together yields $p^* \models v$.
\end{proof}

\begin{corollary}
Running the progressive synthesis algorithm with an \emph{abstract} distance function and an \emph{abstract} uniform covering also achieves soundness with respect to the \emph{concrete semantics}, i.e., if the algorithm returns $u^* = (t^\alpha_n \circ \cdots \circ t^\alpha_1)(u)$, then $p^* = (t_n \circ \cdots \circ t_1)(p_0) \models v$.
\end{corollary}

\section{Deferred Proofs}
\label{appendix:proofs}

\subsection{Abstract Semantics of the Mixing Property}
\label{appendix:proofs:properties}

\begin{proof}[Proof of \cref{lemma:linear_abstract_semantics}]

Given two subgraphs $p$ and $q$, where $u = \alpha_M(p)$, we need to show that $q^\alpha_M(u) := \alpha_M(q) \times u$ is an abstract interpretation, i.e., $ \alpha_M(q) \times u \le \alpha_M(q \circ p)$.

Denote the input tensor as $I$, the intermediate output $p(I)$ as $P$, and the output of the subgraph $q(P) = (q \circ p)(I)$ as $O$. We begin by showing that $q^\alpha_M$ is an abstract interpretation for 1-dimensional inputs and outputs, i.e., when $I$, $P$, and $O$ are all 1-dimensional tensors. First we establish the safety of the pairing property, which corresponds to sufficient conditions for when $\alpha_M(q \circ p) > \lx{}$. Note that when both $q$ and $p$ pair their respective inputs and outputs, then so does their composition (since the preimage of $O$ in $P$ is the entire slice, and the preimage of $P$ in $I$ is also the entire slice, hence the preimage of $O$ in $I$ is also the entire slice). Thus, we can safely define $\alpha_M(q) \times v = \lx{}$ if $v = \lx{}$ or $u = \alpha_M(q) = \lx{}$. This yields the first row and first column of \cref{fig:locality_multiply}.

Otherwise, both $u$ and $v$ are greater than \lx{} (i.e., one of \lo{}, \lm{}, or \la{}). When either $u$ or $v$ have all-to-one locality, so does $q(p)$: if $q$ is all-to-one, then the preimage of a single element in $O$ is the entire slice of $P$ (and since $p$ has at least one-to-one locality, the preimage in $I$ is also the full slice); if $p$ is all-to-one, then the preimage of a single element in $O$ is at least an element of $P$, and so tracing back through $p$ yields the preimage is again the full slice of $I$. So we can define $\alpha_M(q) \times v = \la{}$ if either $u = \alpha_M(q) = \la{}$ or $v = \la{}$ (and neither $u$ nor $v$ is \lx{}). This yields the remaining entries of \cref{fig:locality_multiply} in the last row and last column.

Finally, the only situation in which $q \circ p$ has one-to-one locality is when both $q$ and $p$ are one-to-one; the remaining entries must be many-to-one. This completes the definition of $\alpha_M(q) \times v = u * v$ as given in \cref{fig:locality_multiply}.

When the tensors are higher dimensional, to determine whether a dimension $x$ in $I$ and $z$ in $O$ are paired, we need to check whether there exists some dimension $y$ in $P$ which is paired with both $x$ and $z$. The locality is the maximum locality over all paired dimensions $y$ in $P$, which corresponds to the definition of $+$ in \cref{lemma:linear_abstract_semantics}. Hence we see the computation is exactly a matrix multiplication, and $q^\alpha_M(u) := \alpha_M(q) \times u$ is an abstract interpretation as claimed.
\end{proof}

\subsection{Complexity of Progressive Synthesis}
\label{appendix:proofs:prog_syn}

This section culminates in the statement and proof of the formal version of \cref{thm:synthesis_complexity_informal}. We first clarify how ``length'' is measured for a sequence of transformations $(t_1, \ldots, t_n) \in \mathcal{T}^n$. Note that taking the length to be $n$ is not well-defined since $\mathcal{T}$ is the set of all finite length strings $E^*$; in particular, $\tau = t_n \circ \cdots \circ t_1 \in \mathcal{T}$, i.e., any sequence of transformations is always functionally equivalent to a singleton transformation.

Instead, we define a string representation of any sequence of transformations $(t_1, \ldots, t_n)$ to be the string representation of the composed transformation $\tau = t_n \circ \cdots \circ t_1$ (recall that this is defined since $\tau \in \mathcal{T}$, and elements of $\mathcal{T}$ are also strings over $E$). This yields the following equivalence relation over sequences of transformations:

\begin{definition}
Let $t = (t_1, \ldots, t_n)$ and $s = (s_1, \ldots, s_m)$ be two sequences of transformations. Then $t$ and $s$ are \textbf{representationally equivalent}, denoted as $t \sim s$, if their string representations are lexicographically equal.
\end{definition}

Let us fix some arbitrarily lexicographical ordering over $E$. This allows us to uniquely identify a representative:

\begin{definition}
Given an equivalence class $R$ under $\sim$, the \textbf{primitive representative} is the lexicographically first element in $R$ achieving the longest sequence length.
\end{definition}

It is not hard to see that this maximal element consists entirely of primitive operations, hence the name. 

\begin{definition}
For any $t \in \mathcal{T}$, we define the \textbf{length} of $t$, denoted  \textbf{$|t|$}, as the number of transformations in the primitive representative of its equivalence class.
\end{definition}

Under this definition of length, $(t_1, \ldots, t_n) \in \mathcal{T}^n$ and $\tau = t_n \circ \cdots \circ t_1 \in \mathcal{T}$ have the same length.

\begin{proof}[Proof of \cref{thm:prog_syn}]
Soundness and completeness follows from the fact that $d(v_i, v) = 0$ if and only if $\alpha(p_i) \models v$. Progress follows also from the fact that $T$ is a uniform covering and we decrease $d$ by at least $\epsilon$ each step. At each step we consider each transformation $t \in T$, and we do this once for each transformation in the returned sequence. Hence, the total runtime is linear in $|T|$ and the number of transformations in the returned sequence (which is bounded by the length of the returned sequence).
\end{proof}

Next, we provide a formal definition of a recursively consistent synthesis algorithm, which was first introduced informally in \cref{sec:prog_syn_universal}:

\begin{definition}
Let $D$ be a positive integer. Then $A$ is \textbf{recursively consistent at depth $D$} if there is a constant $C \ge 0$ such that for all feasible $p, v$, there exists a pair of transformations $(t_1, t_2) \sim A(p, v)$ with $|t_1| \le D$ satisfying
\begin{enumerate}
    \item $A(t_1(p), v) \sim t_2$; and
    \item $A(t_1(p), v)$ runs in at most $C$ more steps than $A(p, v)$.
\end{enumerate}
\end{definition}

We claim that our definition of a recursively consistent algorithm is relatively natural. Roughly speaking, the synthesis problem has a recursively consistent algorithm $A$ if there are ``stopping points'' along the way for which $A$ can output partial solutions (of bounded length). For instance, large programs are often organized into several files, which themselves are broken down into functions; when tackling a theorem, lemmas can be used to streamline the proof. Hence, focusing on synthesis problems that are solvable via recursively consistent means does not materially limit the scope of our results from a practical standpoint. Note that recursive consistency is not just a condition on the inherent difficulty of the synthesis task (as in standard computational complexity) but also the difficulty of computing a solution under the semantics of the language of transformations $E^*$.

We also need a more nuanced way to quantify the efficiency of synthesis. So far (e.g., in \cref{thm:prog_syn}), we have been using the following measure of efficiency:
\begin{definition}
A synthesis algorithm $A$ is \textbf{output-efficient} if its computational complexity is polynomial in the length of its output, $|A(p, v)|$.
\end{definition}
The main issue with this definition is that achieving a run time that is polynomial in output length is, in some sense, trivial. For instance, suppose that algorithm $A$ runs in time that is exponential in its output, but there exists a transformation $t_{\text{id}} \in \mathcal{T}$ that is the identity transformation, i.e., $t_{\text{id}}(p) = p$ for all $p \in \mathcal{P}$. Then we can make $A$ run in ``polynomial time'' simply by appending an exponentially long sequence of $t_{\text{id}}$ to the end of the original output. More generally, using output efficiency leads to a host of problems that can be summarized as rewarding ``verbose'' algorithms (i.e., those that return longer sequences of transformations).

To avoid such issues, we instead define efficiency with respect to an independent baseline:
\begin{definition}
A \textbf{complexity function} is a function $K : \mathcal{P} \times \mathcal{V} \rightarrow \mathbb{R}^+$.
\end{definition}
\begin{definition}
A synthesis algorithm is \textbf{\emph{K}-efficient} if its computational complexity is polynomially bounded by the complexity function $K$.
\end{definition}
One way to understand this is that the complexity function $K$ measures how ``difficult'' a particular synthesis task $(p, v)$ is. Then a $K$-efficient algorithm scales gracefully with the difficulty of the input instance. Alternatively, the complexity function tells us roughly how much time we are willing to spend on solving a particular instance of a synthesis task. In this case, a synthesis algorithm is $K$-efficient if the amount of time it takes to solve an input instance grows polynomially with the computational budget. Both of these interpretations give an objective way to compare the efficiency of different synthesis algorithms. Note that there is no requirement that the complexity function itself be efficiently computable.

\begin{algorithm}[t]
\small
\caption{Parallel progressive synthesis}
\label{alg:parallel_synthesis}
\begin{algorithmic}[1]
\Require initial program $p$, program property $v$, a distance function $d$, uniform covering $T$ with lower bound $\epsilon$
\Ensure $\not\models$ if $(p, v)$ is infeasible; otherwise, a sequence of transformations $\tau$ such that $\tau(p) \models v$
  \State $d_0 \gets d(p,v)$
  \If{$d_0 = \infty$}
    \State \Return $\not\models$
  \EndIf
  \State $\tau \gets \{ \}$, $p_0 \gets p$
  \For{$i$ = 1, \ldots}
    \If{$d_{i-1} = 0$}
        \State \Return $\tau$
    \EndIf
    \State $t_i \gets \textsc{ParallelSearch}\big(T; \lambda t.\big(d(t(p_{i-1}), v) + \epsilon \le d_{i-1}\big)\big)$ \label{alg:parallel_synthesis:return_first}
    \State $p_i \gets t_i(p_{i-1})$
    \State $d_i \gets d(p_i, v)$
    \State $\tau \gets (t_1, \ldots, t_i)$
  \EndFor
\end{algorithmic}
\end{algorithm}

We now introduce a new progressive synthesis algorithm that is designed for $K$-efficiency (rather than output efficiency). \cref{alg:parallel_synthesis} presents the pseudocode for this variant, which we call \textbf{parallel progressive synthesis}. The main difference between the greedy and parallel variants of progressive synthesis is the use of the subroutine \textsc{ParallelSearch} on line~\ref{alg:parallel_synthesis:return_first} in \cref{alg:parallel_synthesis}. $\textsc{ParallelSearch}(T;\textsc{Cond})$ takes two arguments: the first argument $T$ is a set, and the second argument \textsc{Cond} is a conditional that takes an element of $T$ and returns either True or False. \textsc{ParallelSearch} then initializes one instance of \textsc{Cond} for each element $t \in T$, and begins executing \textsc{Cond} \emph{in parallel, one step at a time}. Every time an instance of \textsc{Cond} finishes computing, we check its output value and return if the condition is satisfied. Then the run time of \textsc{ParallelSearch} is bounded by the size of the set $T$ multiplied by the minimum amount of time it takes to verify a satisfying $t \in T$.

More explicitly, define $S(f; x)$ to be the total number of steps needed to compute the function $f$ on the input $x$. Define the set of transformations that decrease the distance by at least $\epsilon$ as $T^\downarrow = \{ t \in T \ |\ d(t(p_i), v) + \epsilon \le d_i\}$. Then each call to \textsc{ParallelSearch} on line~\ref{alg:parallel_synthesis:return_first} in \cref{alg:parallel_synthesis} terminates after exactly $|T|\min_{t \in T^\downarrow} S(d; (t(p_i), v))$ steps. We say that \textsc{ParallelSearch} returns the \emph{fastest continuation}:
\begin{align}
    t_i = \argmin_{t \in T^\downarrow} S(d; (t(p), v))
\end{align}

We can compare this with the analogous greedy search subroutine from \cref{alg:synthesis} (greedy progressive synthesis), which always takes the \emph{most progressive continuation}:
\begin{align}
    t_i = \argmin_{t \in T} d(t(p), v)
\end{align}
In contrast to \textsc{ParallelSearch}, the run time of this line is bounded \emph{from below} by the \emph{maximum} amount of time it takes to compute $d(t(p), v)$ for any $t \in T$. This explains why we need to use the \textsc{ParallelSearch} function, since we have no control over the worst-case performance of $d$.

We now state the first version of our main result.

\begin{theorem}
\label{thm:parallel_synthesis_complexity}
Given a synthesis task specified by a space of programs $\mathcal{P}$; a program property $\Pi = (\mathcal{V}, \le, \alpha)$; and a set of primitive transformations $E$; for all complexity functions $K(p, v)$, the following are equivalent:
\begin{enumerate}
    \item There exists a sound, complete, $K$-efficient, and recursively consistent synthesis algorithm $A$.
    \item There exists a distance function $d$ with a uniform covering, such that the parallel progressive synthesis algorithm is sound, complete, and $K$-efficient.
\end{enumerate}
\end{theorem}

\begin{proof}
$(2)\implies(1)$ follows immediately from the fact that progressive synthesis is recursively consistent (with depth 1). The other direction requires us to demonstrate that the existence of a recursively consistent synthesis algorithm also implies that the existence of a viable distance function $d$ for progressive synthesis.

Then let us assume that $A$ is a sound, complete, $K$-efficient, and recursively consistent synthesis algorithm with depth $D$. The first step is to construct the distance function $d$. We define the following distance function:
\begin{align}
    d(p, v) &:= 
    \begin{cases}
    \infty, &\text{if } A(p, v) \text{ returns } \not\models\\
    |A(p, v)| &\text{otherwise}
    \end{cases} \\
\end{align}
By the soundness and completeness of $A$, $d$ satisfies the definition of a distance function. 

Next, define $T$ to to be the set of all strings of length at most $D$, formed from the alphabet $E$. Then $T$ is a uniform covering with $\epsilon = 1$ since $A$ is recursively consistent: given $A(p, v) = (t_1, \ldots, t_n)$, there is some $(s_1, s_2) \sim (t_1, \ldots, t_n)$ with $s_1 \in T^\downarrow \subseteq T$ such that $d(p, v) = |A(p, v)|$ and $d(s_1(p), v) = |A(s_1(p), v)| = |s_2| \le |A(p, v)| - 1$. So by \cref{thm:prog_syn}, the parallel progressive synthesis algorithm is sound and complete.

It remains to show that parallel progressive synthesis is $K$-efficient given the distance function $d$ constructed above. Denote the output $A(p, v) = (t_1, \ldots, t_n)$. We will prove the following two statements about the parallel progressive synthesis algorithm:
\begin{enumerate}
    \item The total number of iterations is at most $n$.
    \item Within each iteration, the call to \textsc{ParallelSearch} is polynomially bounded by
    \begin{align}
        |T|(S(A; (p, v) + C * i)
    \end{align}
    where $C$ is a constant.
\end{enumerate}

Combining the two estimates yields an upper bound for the computational complexity of
\begin{align}
\label{thm:parallel_synthesis_complexity:total_runtime}
    n|T|(S(A; (p, v)) + C * n)
\end{align}
for the parallel progressive synthesis algorithm. Since $A$ is $K$-efficient by assumption, $S(A; (p, v))$ is polynomially bounded by $K(p, v)$. Similarly, the original output length $n = |A(p, v)|$ is upper bounded by the run time of $A$, and hence is also polynomially bounded by $K(p, v)$. Thus, the total computational complexity of the parallel progressive synthesis algorithm is polynomially bounded by $K(p, v)$, as claimed.

It remains to prove the two claims. The first one follows from the fact that \textsc{ParallelSearch} is guaranteed to return a transformation that reduces the distance by at least 1, and since the distance is initialized to $d_0 = d(p, v) = |A(p, v)| = n$, the loop terminates after at most $n$ iterations. To prove the second invariant, consider the $i^{th}$ iteration of the loop. By the definition of recursive consistency, there exists some $s_i \in T^\downarrow$ such that $A(s_i(p_{i-1}), v)$ runs in at most $C$ more steps than $A(p_{i-1}, v)$. Inducting backward on $i$ we see that $A(s_i(p_{i-1}), v)$ runs in at most $C * i$ more steps than $A(p_0, v)$. Furthermore, \textsc{ParallelSearch} is guaranteed to return $t_i \in T^\downarrow$ such that $d(t_i(p), v)$ takes at most as many steps as $d(s_i(p), v)$. Hence, \textsc{ParallelSearch} takes at most $|T|S(A; (p, v) + C * i)$ steps in the $i^{th}$ iteration.
\end{proof}

Note that there is a hidden constant in the total run time of the parallel progressive synthesis algorithm in \cref{thm:parallel_synthesis_complexity:total_runtime} that depends the size of the covering $|T| = O(|E|^D + D)$ (the addition of $D$ makes explicit the dependence on depth when $|E| = 1$). Our proof relies crucially on the depth $D$ of the recursive synthesis algorithm in order to construct the covering set $T$ of all transformations of length at most $D$. We also see that \cref{thm:synthesis_complexity_informal} is the informal statement of the special case when the depth $D = 1$. 

The final result of this section presents a different approach based on universal search \citep{levin1973universal}. Our objective is to define a version of progressive synthesis that does not need to be provided with a covering set $T$. The entire problem is then reduced to defining a suitable distance function $d$.

The main challenge without \emph{a priori} access to a covering set $T$ is that we can no longer rely on \textsc{ParallelSearch}, which essentially searches over a known finite set guaranteed to contain a fast solution. Instead, we will define a new subroutine \textsc{UniversalSearch} that only requires the existence of a fast solution of some (unknown) finite length.

$\textsc{UniversalSearch}(E, \textsc{Cond})$ takes two arguments: an alphabet $E$ and a conditional \textsc{Cond} that takes as input any string $t \in E^*$ and returns either True or False. The execution of \textsc{UniversalSearch} then proceeds in phases, starting from phase $i = 0$. During the $i^{th}$ phase, \textsc{UniversalSearch} evaluates \textsc{Cond} on all strings $t \in E^*$ of length at most $i$, up to a total of $\max(i, |E|^i)$ steps per instance in parallel (i.e., first by running all $t \in \mathcal{T}$ of length \emph{equal} to $i$ for up to $\max(i-1, |E|^{i-1})$ steps in parallel, then by running all $t \in \mathcal{T}$ of length \emph{at most} $i$ for up to a total of $\max(i, |E|^{i})$ steps in parallel). The rest of the strategy is the same as \textsc{ParallelSearch}: whenever an instance of \textsc{Cond} completes, we return if the condition is satisfied. \Cref{alg:universal_synthesis} presents the resulting synthesis algorithm, which we call \textbf{universal progressive synthesis}.

\begin{algorithm}[t]
\small
\caption{Universal progressive synthesis}
\label{alg:universal_synthesis}
\begin{algorithmic}[1]
\Require initial program $p$, program property $v$, a distance function $d$, primitives $E$
\Ensure $\not\models$ if $(p, v)$ is infeasible; otherwise, a sequence of transformations $\tau$ such that $\tau(p) \models v$
  \State $d_0 \gets d(p,v)$
  \If{$d_0 = \infty$}
    \State \Return $\not\models$
  \EndIf
  \State $\tau \gets \{ \}$, $p_0 \gets p$
  \For{$i$ = 1, \ldots}
    \If{$d_{i-1} = 0$}
        \State \Return $\tau$
    \EndIf
    \State $t_i \gets \textsc{UniversalSearch}\big(E; \lambda t.\big(d(t(p_{i-1}), v) < d_{i-1}\big)\big)$
    \State $p_i \gets t_i(p_{i-1})$
    \State $d_i \gets d(p_i, v)$
    \State $\tau \gets (t_1, \ldots, t_i)$
  \EndFor
\end{algorithmic}
\end{algorithm}

\begin{theorem}
Given a synthesis task specified by a space of programs $\mathcal{P}$; a program property $\Pi = (\mathcal{V}, \le, \alpha)$; and a set of primitive transformations $E$; for all complexity functions $K(p, v)$, the following are equivalent:
\begin{enumerate}
    \item There exists a sound, complete, $K$-efficient, and recursively consistent synthesis algorithm $A$.
    \item There exists a distance function $d$ such that the universal progressive synthesis algorithm is sound, complete, and $K$-efficient.
\end{enumerate}
\end{theorem}

\begin{proof}
This proof is nearly the same as the proof of \cref{thm:parallel_synthesis_complexity}, except that we need to analyze the performance of the subroutine \textsc{UniversalSearch} instead of \textsc{ParallelSearch}. In particular, we will use the same distance as in the proof of \cref{thm:parallel_synthesis_complexity}, namely,
\begin{align}
    d(p, v) &:= 
    \begin{cases}
    \infty, &\text{if } A(p, v) \text{ returns } \not\models\\
    |A(p, v)| &\text{otherwise}
    \end{cases} \\
\end{align}

Now assume that  \textsc{UniversalSearch} returns a solution $t$ of length $D$ that takes $S$ steps to verify. We can upper bound the run time of \textsc{UniversalSearch} as follows.

First, if $|E| = 1$, then $\textsc{Cond}(t)$ begins execution during phase $D$ and completes execution during phase $D + S$. At the end of phase $D + S$, there are exactly $D + S$ instances of \textsc{Cond} which have begun execution, each of which has been run for at most $D + S$ steps. This yields an upper bound for the total run time of at most $(D + S)^2$.

Otherwise, if $|E| > 1$, then $\textsc{Cond}(t)$ begins execution in phase $D$ and completes execution by the end of phase $D + \lceil \log_E(S) \rceil$. At this point there are exactly
\begin{align}
    \sum_{i = 0}^{D + \lceil \log_E(S) \rceil} |E|^i
\end{align}
instances of \textsc{Cond} that have begun execution, which can be bounded from above by
\begin{align}
    \sum_{i = 0}^{D + \lceil \log_E(S) \rceil} |E|^i
    & \le 2 |E| ^ {D + \lceil \log_E(S) \rceil} \\
    & \le 2 S |E| ^ {D + 1}
\end{align}
The total number of sequential steps taken by any instance of \textsc{Cond} is upper bounded by
\begin{align}
    |E| ^ {D + \log_E(S)} = S |E| ^ D
\end{align}
Putting all this together, the total number of steps taken by \textsc{UniversalSearch} is at most
\begin{align}
    2 S^2 |E| ^ {2D + 1} + (D + S)^2 = O(S^2(|E|^{2D + 1} + D^2))
\end{align}

In particular, we see that \textsc{UniversalSearch} pays at most a quadratic overhead over \textsc{ParallelSearch}, except that here $D$ refers to the \emph{unknown} finite depth of the recursively consistent algorithm $A$. We omit the remainder of this proof as it is otherwise identical to the proof of \cref{thm:parallel_synthesis_complexity}.
\end{proof}

\subsection{Product Properties and Monotonic Transformations}
\label{appendix:proofs:multiple_properties}

The main result of this section is a proof of \cref{thm:prog_syn_multi}. We begin with a method of combining two properties by taking their product.

\begin{definition}
Let $\Pi_1 = (\mathcal{V}_1, \le_1, \alpha_1)$ and $\Pi_2 = (\mathcal{V}_2, \le_2, \alpha_2)$ be two properties. The \textbf{product property} $\Pi = (\mathcal{V}, \le, \alpha)$, denoted $\Pi = \Pi_1 \times \Pi_2$, takes values $\mathcal{V} = \mathcal{V}_1 \times \mathcal{V}_2$, with the partial order $(u_1, u_2) \le (v_1, v_2)$ if and only if $u_1 \le_1 v_2$ and $u_2 \le_2 v_2$, and abstraction function $\alpha(p) := (\alpha_1(p), \alpha_2(p))$.
\end{definition}

Here are some useful facts about the product property:

\begin{lemma}
\label{lemma:product}
Let $\Pi_1 = (\mathcal{V}_1, \le_1, \alpha_1)$ and $\Pi_2 = (\mathcal{V}_2, \le_2, \alpha_2)$ be two properties, and let $d_1$ and $d_2$ be distance functions on $\Pi_1$ and $\Pi_2$, respectively. Denote their product property as $\Pi = \Pi_1 \times \Pi_2 = (\mathcal{V}, \le, \alpha)$, and define a function $d : \mathcal{P} \rightarrow \mathcal{V}$ as $d(p, (u, v)) = d_1(p, u) + d_2(p, v)$. Let $T \subseteq \mathcal{T}$ be a set of transformations. We have the following facts:
\begin{enumerate}
    \setcounter{enumi}{-1}
    \item $\Pi$ is a program property (cf. \cref{def:property}).
    \item $p \models u \in \mathcal{V}_1$ and $p \models v \in \mathcal{V}_2$ if and only if $p \models (u,v) \in \mathcal{V}$.
    \item $d$ is a distance function for $\Pi$ (cf. \cref{def:distance}).
    \item \label{lemma:product:item:monotonic} If $T$ is monotonic with respect to $d_1$ and $d_2$, then $T$ is also monotonic with respect to $d$.
    \item Let $T_1$ and $T_2$ be monotonic with respect to $d_2$ and $d_1$, respectively, and define $T = T_1 \cup T_2$.
    \begin{enumerate}
    \item \label{lemma:product:item:monotonic_covering} If $T_1$ is a covering for $\Pi_1$ (with respect to $d_1$), and $T_2$ is a covering for $\Pi_2$ (with respect to $d_2$), then $T$ is also a covering for $\Pi$ (with respect to $d$).
    \item \label{lemma:product:item:monotonic_covering_uniform} Furthermore, if $T_1$ and $T_2$ are both \emph{uniform} coverings, then so is $T$.
    \end{enumerate}
\end{enumerate}
\end{lemma}
From the first two facts, we see that for a synthesis task with multiple properties, it suffices to solve the synthesis task on a single product property. The next statement gives an easy way to construct a new distance function for the product property, and the following statement says any set which was monotonic for the individual distance functions is also monotonic for this new distance function. We omit these proofs as they follow immediately from the definitions.

\cref{lemma:product} (\ref{lemma:product:item:monotonic_covering}) and (\ref{lemma:product:item:monotonic_covering_uniform}) give sufficient conditions for constructing coverings for the product property. Note that these statements are not true in general without monotonicity---in fact, even if $T$ is a \emph{uniform} covering for both $\Pi_1$ and $\Pi_2$, the following statements about the product property $\Pi = \Pi_1 \times \Pi_2$ are \emph{false} in general:
\begin{enumerate}
    \item $T$ is a covering for $\Pi$.
    \item if $T$ is a covering for $\Pi$, then $T$ is a uniform covering for $\Pi$.
\end{enumerate}
Hence we see that monotonicity is the key to progressive synthesis over multiple properties when taking the product.

\begin{proof}[Proof of \cref{lemma:product}]
Given $p \in \mathcal{P}, u \in \mathcal{V}_1, v \in \mathcal{V}_2$, we first show that if $0 < d(p, (u,v)) < \infty$, there exists $t \in T$ such that $d(t(p), (u, v)) < d(p, (u,v))$. Since $d(p, (u,v)) = d_1(p, u) + d_2(p, v) > 0$, and by the definition of a distance function $d_1$ and $d_2$ are nonnegative, it follows that at least one of $d_1$ or $d_2$ is also strictly positive.

Assume without loss of generality that $d_1(p, u) > 0$. If $T_1$ is a covering, it follows that there exists $t \in T_1$ such that $d_1(t(p), u) < d_1(p, u)$. Since $d_2$ is monotonic with respect to $T_1$, $d_2(t(p), v) \le d_2(p, v)$. Adding these together yields that
\begin{align}
    d_1(t(p), u) + d_2(t(p), v) < d_1(p, u) + d_2(p, v)
\end{align}
i.e.,
\begin{align}
    d(t(p), (u, v)) < d(p, (u, v))
\end{align}
Hence, $T$ is a covering for $d$, which proves \cref{lemma:product} (\ref{lemma:product:item:monotonic_covering}).

Furthermore, if $T_1$ is actually a \emph{uniform} covering for $d_1$ with lower bound $\epsilon_1$, then there exists $t \in T_1$ such that $d_1(t(p), u) + \epsilon_1 < d_1(p, u)$. Hence,
\begin{align}
    d_1(t(p), u) + d_2(t(p), v) + \epsilon_1 < d_1(p, u) + d_2(p, v)
\end{align}
i.e.,
\begin{align}
    d(t(p), (u, v)) + \epsilon_1 < d(p, (u, v))
\end{align}
So $T$ is also a uniform covering for $d$ (with lower bound equal to the minimum of the lower bounds for $d_1$ and $d_2$), which proves \cref{lemma:product} (\ref{lemma:product:item:monotonic_covering_uniform}).
\end{proof}

\begin{proof}[Proof of \cref{thm:prog_syn_multi}]
We will prove this by inducting on the number of properties $N$. If $N = 1$, then the statement is trivially true. Otherwise, assume the result holds for $N = M$; we will show it holds for $M+1$. 

We first prove the case when each $T_i$ is a covering for $d_i$, respectively. By the inductive hypothesis, $T^M = \cup_{i=1}^M T_i$ is a covering for the distance function $d^M(p, S^M) = \sum_{i=1}^M d_i(p, v_i)$ with respect to the product property $\Pi^M = \Pi_1 \times \cdots \times \Pi_M$. By assumption, $T^M$ is monotonic with respect to $d_{M+1}$, and by \cref{lemma:product} (\ref{lemma:product:item:monotonic_covering}), $T_{M+1}$ is monotonic with respect to $d^M$. Hence, we can apply \cref{lemma:product} (\ref{lemma:product:item:monotonic_covering}) to conclude that
\begin{align}
    T^{M+1} = T^M \cup T_{M+1} = \bigcup_{i=1}^{M+1} T_i
\end{align}
is a covering for the distance function 
\begin{align}
    d^{M+1} = d^M + d_{M+1} = \sum_{i=1}^{M+1} d_i
\end{align}
with respect to the product property 
\begin{align}
    \Pi^{M+1} = \Pi^M \times \Pi_{M+1} = \Pi_1 \times \cdots \times \Pi_{M+1}
\end{align}
as desired.

If each $T_i$ is a uniform covering for $d_i$, then by the same argument, we can conclude via \cref{lemma:product} (\ref{lemma:product:item:monotonic_covering_uniform}) that $T^{M+1}$ is a uniform covering for $d^{M+1}$ with respect to $\Pi^{M+1}$.
\end{proof}

\subsubsection{Weak uniform coverings}

We conclude this section with a short technical note. Recall that for a set $T$ to be a uniform covering of a distance function $d$, there must be some constant $\epsilon > 0$ such that if $d(p, v) > 0$, then we can find a $t \in T$ such that $d(t(p), v) + \epsilon \le d(p, v)$. A consequence of this definition is that $d(p, v) < \epsilon$ implies that $d(p, v) = 0$. It follows that if the distance function $d$ can get arbitrarily small, it cannot be uniformly covered.

We address this by introducing a weaker notion of uniform covering which can be applied to distance functions that vanish:

\begin{definition}
A covering $T$ is a \textbf{weak uniform covering} if there exists a constant $\epsilon > 0$ such that for all programs $p \in \mathcal{P}$ and properties $v \in \mathcal{V}$ with $d(p, v) > 0$, there exists $t \in T$ such that 
\begin{align}
    d(t(p), v) + \min(\epsilon, d(p, v)) \le d(p, v)
\end{align}
\end{definition}

For a single program property, a weak uniform covering $T$ is almost as good as a uniform covering in that if $d(p, v) \ge \epsilon$, then searching over $T$ achieves the same uniform lower bound $\epsilon$; but if $d(p, v) < \epsilon$, although the improvement achieved by searching over $T$ is not uniformly bounded from below, it is still guaranteed to reach a distance of 0.

For multiple program properties, a subtle issue remains in that \cref{lemma:product} (\ref{lemma:product:item:monotonic_covering}) and (\ref{lemma:product:item:monotonic_covering_uniform}) do not hold for weak uniform coverings, i.e., if $T_1$ and $T_2$ are \emph{weak} uniform coverings, then $T = T_1 \cup T_2$ is only guaranteed to be a covering (rather than a weak uniform covering). One way to address this that $T^2$, the set of all \emph{pairs} of transformations from $T$, is a weak uniform covering for $d = d_1 + d_2$; however, this introduces a quadratic factor to the dependence on the number of transformations in the respective covering sets (and in general, for the product of of $n$ program properties, we would need to search over transformations in $T^n$, which is undesirable).

Fortunately, there is a simple way to convert a weak uniform covering into a uniform covering that resolves this issue:

\begin{lemma}
If $T$ is a weak uniform covering for a distance function $d$, then there exists another function $d'$ such that $T$ is a uniform covering for $d'$.
\end{lemma}

\begin{proof}
Assume that $T$ is a weak uniform covering with lower bound $\epsilon$. Define the following function:
\begin{align}
d'(p, v) &:= 
    \begin{cases}
    0, &\text{if } d(p, v) = 0\\
    d(p, v) + \epsilon &\text{otherwise}
    \end{cases}
\end{align}
Clearly, $d'$ is a distance function, and $T$ is a uniform covering for $d'$.
\end{proof}

\subsection{Progressive Synthesis for NAS}
\label{appendix:proofs:prog_syn_nas}

Before presenting the next few results characterizing the depth, mixing, and shape properties, we begin with a simple lemma that provides sufficient conditions for a covering $T$ to be a uniform covering.

\begin{lemma}
If the image of the distance function $d$ is a discrete set (with the usually topology on $\mathbb{R}$), then any covering $T$ is also a uniform covering.
\end{lemma}

We omit the proof as it follows immediately from the definitions. Since all our distance functions take integer values, for the remainder of this section, it suffices to show that the set of simple primitives $E_s$ is a covering.

The next three results culminate in the proof of \cref{lemma:preserve}, which says that the set of primitive operations are monotonic for the mixing property.

\begin{lemma}
\label{lemma:semiring}
The set $\{\lx{}, \lo{}, \lm{}, \la{}\}$ along with the operations $+$ and $*$ as defined in Lemma \ref{lemma:linear_abstract_semantics} forms a semi-ring.
\end{lemma}

\begin{proof}
The additive identity $0$ is the element \lx{}, and the $\max$ operator is commutative and reflexive. The multiplicative identity element is \lo{}, and the annihilating element is \lx{}. Finally, $*$  distributes over $+$ from both the right and the left.
\end{proof} 

\begin{corollary}
\label{cor:matrix_semiring}
The set of $m$-by-$m$ mixing properties form a semi-ring with the usual matrix multiplication and element-wise addition.
\end{corollary}

In other words, the abstract interpretation in Lemma \ref{lemma:linear_abstract_semantics} is just multiplication in the semi-ring of square matrices over $\{\lx{}, \lo{}, \lm{}, \la{}\}$. Note that the multiplicative identity element is the usual identity matrix $I_m$, where the diagonal is $O$ and all other entries are $x$. As a mixing property, the identity corresponds to an element-wise operation (e.g., a ReLU).

\begin{proof}[Proof of \cref{lemma:preserve}] First, all (shape-preserving) operations preserve at least the diagonal, i.e., for all $e \in E_s$, we have that $\alpha_M(e) + I = \alpha_M(e)$. Let $e_{n1} \circ \cdots \circ e_{11} = t_1 \in \mathcal{T}_s$ and $e_{m2} \circ \cdots \circ e_{12} = t_2 \in \mathcal{T}_s$. Then following the usual rules of matrix multiplication and addition, for any $U \in \mathcal{V}_M$ we have
\begin{align}
    (t_2)^\alpha_M((t_1)^\alpha_M(U))
                     &= \alpha_M(e_{m2}) \cdots \alpha_M(e_{12}) \,\times\, \alpha_M(e_{n1}) \cdots \alpha_M(e_{11}) \,\times\, U \\
                     &= (I + \alpha_M(e_{m2})) \cdots (I + \alpha_M(e_{12})) \,\times\, (I + \alpha_M(e_{n1})) \cdots (I + \alpha_M(e_{11})) \,\times\, U \\
                     &\ge (I + \alpha_M(e_{m2}) \cdots \alpha_M(e_{12})) \,\times\, (I + \alpha_M(e_{n1}) \cdots \alpha_M(e_{11})) \,\times\, U
\end{align}
At this point, we either drop the second term on the right hand side to get:
\begin{align}
(t_2)^\alpha_M((t_1)^\alpha_M(U))
                     &\ge (I + \alpha_M(e_{m2}) \cdots \alpha_M(e_{12})) \,\times\, U \\
                     &\ge \alpha_M(e_{m2}) \cdots \alpha_M(e_{12}) \,\times\, U \\
                     &= (t_2)^\alpha_M (U)
\end{align}
or drop the first term on the right hand side to get:
\begin{align}
(t_2)^\alpha_M((t_1)^\alpha_M(U))
                     &\ge (I + \alpha_M(e_{n1}) \cdots \alpha_M(e_{11})) \,\times\, U \\
                     &\ge \alpha_M(e_{n1}) \cdots \alpha_M(e_{11}) \,\times\, U \\
                     &= (t_1)^\alpha_M (U)
\end{align}
Putting these together yields
$(t^\alpha_2 \circ t^\alpha_1) (U) \ge \max(t^\alpha_2 (U), t^\alpha_1 (U))$ 
as desired.
\end{proof}

\begin{proof}[Proof of \cref{thm:progressive_nas}]
We will apply \cref{thm:prog_syn_multi} to the depth, mixing, and shape properties.
From \cref{thm:depth_monotone,thm:mixing_monotone}, we have that the set of simple primitives $E_s$ is a uniform monotone covering of both the depth and mixing properties, respectively.
From \cref{thm:shape_uniform}, we have that $E_s$ is a uniform covering of the shape property.
Hence, all we need to show is that $E_s \cap M_S$, where $M_S$ is the set of monotonic transformations for the shape property, is also a uniform covering of the depth and mixing properties.

Note that $M_S$ consists exactly of the set of operations that preserve the shape of their input tensor. Only 3 operations cause a change in the shape of their input tensor: (1) pooling operations, (2) dense layers with a different number of output features, and (3) convolution layers with a different number of output features. For the mixing property, the remaining variants that do not change the number of output features are equivalent in the abstract semantics but still preserve their output shapes. 
Hence, removing (2) and (3) still leaves a uniform covering. For the pooling operations, these introduce many-to-one locality between the spatial dimensions,
but this can be achieved by a convolution layer (that also preserves the shape).
Hence, $E_s \cap M_S$ is also a uniform covering of the mixing property. For the depth property, removing these operations clearly still leaves a uniform covering.
\end{proof}

\end{document}